%% file: AnonymousSubmission2027.tex
\documentclass[letterpaper]{article} 
\usepackage{aaai2027}
\usepackage[hyphens]{url}  
\usepackage{graphicx} 
\usepackage{natbib}  
\usepackage{caption} 
\usepackage{algorithm}
\usepackage{algpseudocode}

\usepackage{newfloat}
\usepackage{listings}
\DeclareCaptionStyle{ruled}{labelfont=normalfont,labelsep=colon,strut=off} 
\floatstyle{ruled}
\newfloat{listing}{tb}{lst}{}
\floatname{listing}{Listing}

\usepackage{booktabs}

\usepackage{microtype}
\usepackage{multirow}
\usepackage{arydshln}
\usepackage{adjustbox}
\usepackage{subcaption}
\usepackage{amsthm}
\newtheorem{theorem}{Theorem}
\usepackage{booktabs}
\usepackage{array}
\usepackage{tablefootnote}
\usepackage{amsmath}
\usepackage{amssymb}
\usepackage{scalerel}
\usepackage{siunitx}
\usepackage{tabularx}
\usepackage[table]{xcolor}
\usepackage{soul}
\usepackage{makecell}

\usepackage{inconsolata}
\usepackage{caption}
\usepackage{enumitem}

\nocopyright
\title{\includegraphics[width=0.04\textwidth]{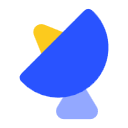}\textsc{ValueGraph}: Value-Signal Guided Graph Pre-training for Contextualized \\ User Representation}

\author {
    Yitong Han\textsuperscript{\rm 1},
    Wei Gao\textsuperscript{\rm 1},
    Yi Zhao\textsuperscript{\rm 1},
    Prasanta Bhattacharya\textsuperscript{\rm 2},\\
    Fengzhu Zeng\textsuperscript{\rm 1},
    Mohammad Amanlou\textsuperscript{\rm 1}
}
\affiliations {
    \textsuperscript{\rm 1}School of Computing and Information Systems, Singapore Management University, Singapore\\
    \textsuperscript{\rm 2}Institute of Advanced Intelligence and Computing, A*STAR, Singapore\\
    \{ythan,weigao,yizhao,fzzeng,mohammada\}@smu.edu.sg,~~ 
    prasanta\_bhattacharya@a-star.edu.sg
}

\begin{document}

\maketitle

\begin{abstract}
Value signals are aggregated user-level moral representations that capture users' inferred value-related tendencies from their online discourse. User behavior on social media is shaped not only by what users say or whom they interact with, but also by the value signal through which they express attitudes. Existing user representation methods largely miss this value-relevant dimension. We propose \textbf{\textsc{ValueGraph}}, a graph pre-training framework that uses automatically inferred moral-value signals as noisy auxiliary signals for contextualized user representation. From post-reply graphs, \textsc{ValueGraph} learns semantic and structural representations and further aligns users through relative value similarity with contrastive and clustering objectives. Rather than treating inferred values as gold psychological labels, \textsc{ValueGraph} uses them as soft constraints for representation learning. Experiments on stance detection and twitter bot detection show consistent gains over strong text-based, graph-based, and text-only LLM baselines, highlighting value-signal guidance as a useful inductive bias for socially informed user modeling.\footnote{Code is released at \url{https://github.com/HanYiton/ValueGraph}}
\end{abstract}


\section{Introduction}\label{introduction}
Understanding user behavior is central to social media analytics, recommendation systems, and personalized AI. User embedding models map user content and interactions to dense representations for downstream tasks such as profiling, recommendation, stance detection, and social influence analysis, and must generalize to unseen users and evolving social contexts~\citep{pan2019social}. 

Most existing methods learn user representations from text~\citep{benton-etal-2016-learning,ding-etal-2017-multi}, visual content~\citep{do2018twitter}, or network structure~\citep{rahimi-etal-2015-twitter,pick2022stem}, and optimize them with self-supervised or graph-based objectives~\citep{Perozzi2014DeepWalkOL,Sun2020MultiStageSL,Grover2016node2vecSF,Donnat2018LearningSN}. These objectives often assume that users with similar content exposure or graph neighborhoods should have similar embeddings. However, similarity in interaction does not necessarily indicate similarity in the underlying factors that drive user behaviors. Users may engage with the same event because of shared exposure, controversy, or platform dynamics while expressing substantially different motivations and attitudes. When social contexts shift, such surface-level similarity can become brittle~\citep{Zhao2021DataAF,Hafidi2020GraphCLCS,hassani2020contrastive,Sun2020InfoGraphUA}. For example, news recommendation systems that rely on interaction-based similarity may mistake shared exposure or controversy-driven interactions for shared preferences or values. Users who engage with the same event may hold opposing attitudes, but such spurious behavioral associations can cause the system to repeatedly recommend similar content to these users, reinforcing homogeneous information exposure and potentially amplifying polarization or extremist dynamics~\citep{10.1145/3351095.3372879}. These limitations suggest that effective user representations should capture not only observable behaviors, such as what users say or where they interact, but also latent factors underlying these behaviors. 

To address these challenges, we adopt Moral Foundations Theory (MFT)~\citep{haidt2007moral} as our theoretical foundation. MFT provides a framework for understanding moral judgments reflected in social discourse and defines interpretable moral dimensions that have been widely studied in computational social science~\citep{GRAHAM201355}. Building upon recent advances in text-based moral value prediction, we derive 10-dimensional user-level moral feature vectors by aggregating post-level moral scores output from MoralBERT~\citep{moralbert}, and refer to these vectors as value signals. Without explicitly modeling such value signals, conventional embedding models may struggle to distinguish users who exhibit similar interaction patterns but hold different stances and behavioral tendencies. This motivates us to leverage value signals as an auxiliary supervisory cue to learn more robust user representations.

We therefore propose \textsc{ValueGraph}, a contextualized user representation framework that incorporates inferred value signals to regularize graph pre-training. In this work, contextualized refers to representations learned by jointly modeling users’ textual content and their interaction context in conversation graphs.
Concretely, \textsc{ValueGraph} operates on a post-reply graph and learns user representations in two stages. First, it uses masked graph autoencoding to learn semantic and structural post representations from conversation graphs. Second, it aggregates post embeddings into user embeddings and constructs positive and negative user pairs according to inferred value similarity, which is derived from psychology and computational ethics literature~\citep{moralbert,nguyen2024measuring,guo2023data}. A user-level contrastive loss aligns users with similar inferred value profiles, while a clustering loss encourages coherent and separated regions in the embedding space. These objectives guide the encoder to represent \emph{what} users engage with alongside the value signals associated with their behavior.

We evaluate \textsc{ValueGraph} on stance detection and Twitter bot detection. Results show consistent gains over graph pre-training methods, Pre-trained Language Model (PLM) encoders, inductive GNNs, and text-only LLM baselines, while ablations confirm that the improvements come from jointly modeling semantic, structural, and value-signal information. 
Our main contributions are summarized as follows:
\begin{itemize}
   \item We propose \textsc{ValueGraph}, a value-signal guided graph pre-training framework that leverages inferred value signals as noisy auxiliary supervision for contextualized user representation learning.
   \item We design a hierarchical objective that combines masked graph autoencoding with contrastive learning and clustering to jointly capture semantic, structural, and value information.
   \item We provide a mechanism-level analysis of the proposed objectives, characterizing value-signal similarity preservation and cluster compactness/separation in the learned embedding space.
   \item Experiments show that \textsc{ValueGraph} achieves the best performance among compared user representation methods on both stance detection and bot detection tasks, validating the effectiveness of incorporating value signals into graph-based user representation learning.
\end{itemize}

\section{Related Work}

\paragraph{User Representation.}
Social media user representations map high-dimensional user features into dense embeddings~\citep{pan2019social}. Prior work learns such representations from text~\citep{benton-etal-2016-learning,ding-etal-2017-multi}, visual content~\citep{do2018twitter}, network structure~\citep{rahimi-etal-2015-twitter}, or multi-view fusion~\citep{zhang2017user,10.1145/3178876.3186026,ribeiro2018characterizing}. These methods are effective, but rarely use value-relevant moral signals as explicit inductive bias.

\paragraph{Graph Pre-training.}
GNN pre-training improves representation generalization by exploiting graph structure and neighborhood similarity~\citep{Perozzi2014DeepWalkOL,Sun2020MultiStageSL,Grover2016node2vecSF,Donnat2018LearningSN,Zhang2019ProNEFA,Tang2015LINELI,Zhao2021DataAF}. Later work explores similar-domain and cross-domain pre-training to improve transferability~\citep{Hafidi2020GraphCLCS,hassani2020contrastive,Sun2020InfoGraphUA,Zhu2020DeepGC,Hu2020Strategies,you2020graph,Hu2020GPTGNNGP,Lu2021LearningTP,Qiu2020GCCGC}. However, existing graph pre-training methods rarely model value signal or motivational dimensions in user behavior.

\paragraph{AI and Human Values.}
Recent work detects moral-values in text using transformer-based models~\citep{moralbert,nguyen2024measuring,guo2023data,zangari2025me2}, while value-alignment research focuses on instruction following and preference-based reward modeling~\citep{ouyang2022training,christiano2017rlhf}. Instead of focusing on text-level moral-value detection or language-model alignment, we use automatically inferred value signals as auxiliary signals for graph-based user representation learning.

\section{Problem Formulation} 

\begin{figure*}[!htbp]
    \centering
    \includegraphics[width=0.90\linewidth]{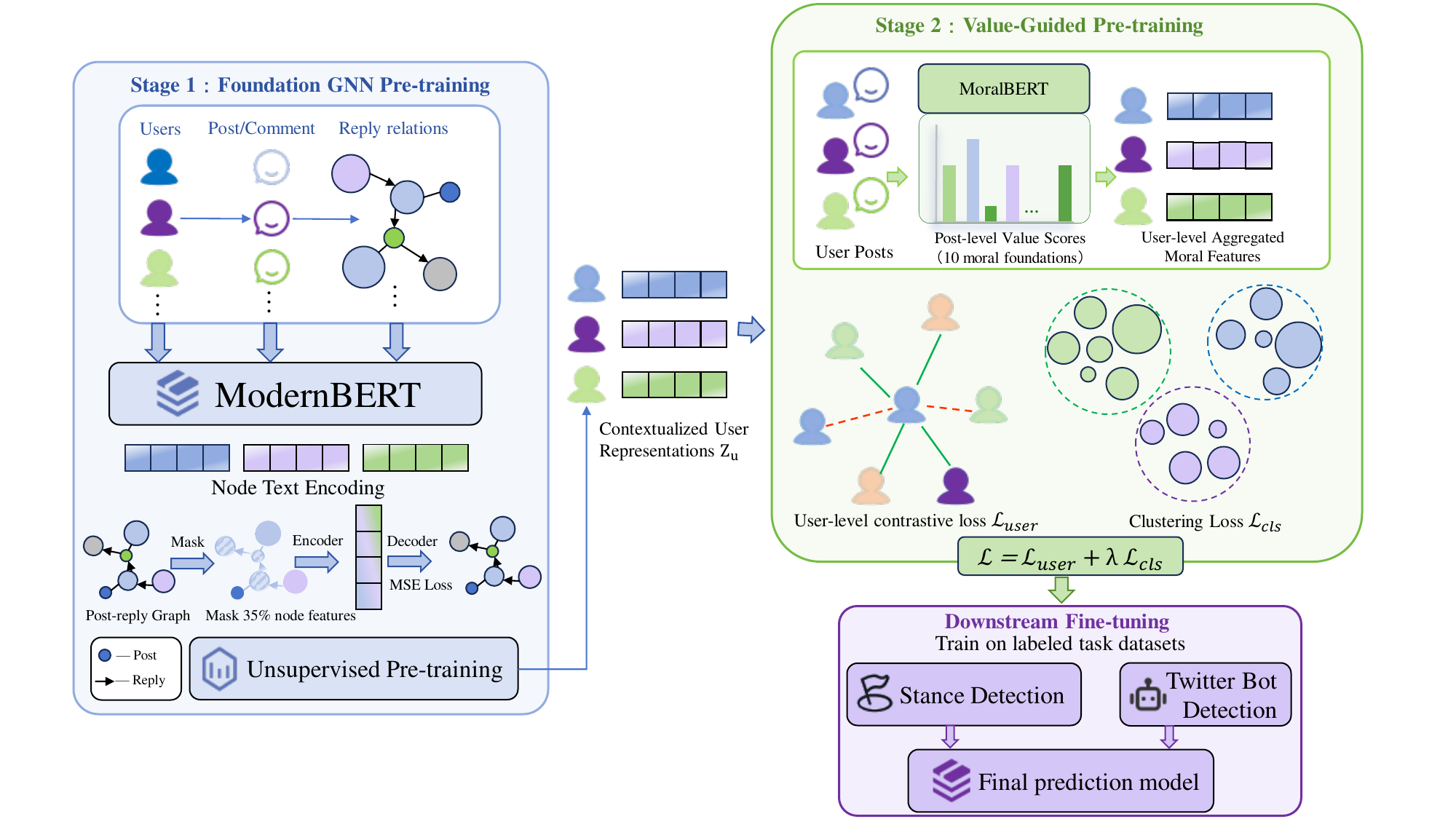}
    \caption{The framework of \textsc{ValueGraph}. In Stage 1, we construct a post-reply graph and perform GraphMAE2-based unsupervised pre-training to learn contextualized post representations by jointly modeling semantic and structural information. The resulting post embeddings are aggregated into user representations.
    In Stage 2, we introduce moral value signals extracted by MoralBERT and optimize user representations through value-guided contrastive learning and clustering objectives. The learned user representations are then used for downstream fine-tuning on stance detection and Twitter bot detection.}
    \label{fig:framework}
\end{figure*}

We are given a dataset $\mathcal{G}=\{\mathcal{G}_i\}_{i=1}^n$ of social conversation graphs. Each graph $\mathcal{G}_i = (\mathcal{V}_i, \mathcal{E}_i)$ represents a conversation thread, where $\mathcal{V}_i$ is the set of posts and $\mathcal{E}_i$ is the set of directed reply edges between posts\footnote{We focus on post-reply relations, since follow/friendship relations are often platform-dependent and unavailable across datasets. This design enables a more general evaluation of value-signal guided graph pre-training.}. Let $\mathcal{U}$ denote all users who authored posts in $\mathcal{G}$. Since each post is authored by exactly one user, we define an authorship mapping $a: \mathcal{V}_i \to \mathcal{U}$, where $u=a(v)$ indicates that post $v$ is authored by user $u$. Let $\mathcal{V}=\bigcup_{i=1}^{n}\mathcal{V}_i$ be the set of all posts. We define $\phi: \mathcal{U} \to 2^\mathcal{V}$ as a post aggregation function, where $\phi(u)=\{v\in\mathcal{V}\mid a(v)=u\}$ is the set of posts authored by $u$.

The goal of user representation learning is to learn a parameterized encoder $f_{\theta}: u \to z_u \in \mathbb{R}^d$ that maps each user $u\in \mathcal{U}$ to a compact embedding $z_u$. The embedding $z_u$ is generated by aggregating semantic information from the user's posts $\phi(u)$ and structural context from post-post reply relations. The resulting representations are expected to capture semantic, relational, and value-relevant behavioral cues that support downstream behavior and content understanding~\citep{pan2019social}, such as stance detection and Twitter bot detection (Section Experiments and Results).

\section{Methodology}
Our \textsc{ValueGraph} consists of two stages: Foundation GNN Pre-training and Value-Guided Pre-training. In the first stage, we pre-train a foundation GNN on the post-reply graph with a masked autoencoding paradigm to obtain noise-resilient representations. In the second stage, we refine these representations using value signals derived from MoralBERT. Figure~\ref{fig:framework} shows an overview of our \textsc{ValueGraph} framework.



\subsection{Stage 1: Foundation GNN Pre-training.}
We first perform unsupervised pre-training of a foundation GNN on the constructed graph dataset. Each post is encoded using ModernBERT~\citep{warner2024smarterbetterfasterlonger}, chosen for robustness to noisy social media text. Reply interactions define graph edges, enabling multi-hop message passing to capture local and global conversational patterns.
Pre-training follows the masked autoencoding paradigm of GraphMAE2~\citep{10.1145/3543507.3583379}. Node features are masked before encoding, and the model reconstructs the original features. This objective encourages noise-resilient representations that serves as a strong initialization for the subsequent value-guided pre-training.


\subsection{Stage 2: Value-Guided Pre-training.}
In this stage, we refine the pre-trained representations with a hierarchical contrastive learning framework. Specifically, we infer a value profile for each user, form contrastive user pairs between users who have similar and dissimilar values, and combine a user-level contrastive objective $\mathcal{L}_{\text{user}}$ with a clustering objective $\mathcal{L}_{\text{cls}}$, which fully exploit the inter-user relations encoded in value similarity and remain robust to the noise in the signal.



\paragraph{Human-Value Grounding for Training Data.}
\label{sec:human_value}

Human values play a central role in shaping social interactions, moral reasoning, and ideological expression. Several frameworks model human values, including Hofstede’s Cultural Dimensions Theory~\citep{hofstede2001culture}, MFT~\citep{haidt2007moral}, and Schwartz’s Theory of Basic Human Values~\citep{schwartz2012refining}. These frameworks have been widely adopted to analyze value-driven language and behavior.  
We adopt MFT as our primary value framework for three reasons. First, it captures moral judgments expressed in everyday language, making it well suited for social media analysis~\citep{johnson2018classification,mooijman2018moralization}. Second, its effectiveness in extracting moral perspectives from text has been empirically validated~\citep{lin2018acquiring,mokhberian2020moral}. Third, it provides structured and interpretable moral dimensions aligned with our objective of modeling value-related signals~\citep{ARAQUE2020105184,zhang2024enhancingstanceclassificationsocial}. More details about MFT are provided in Appendix: Moral Foundations Background. 

To obtain value signals for pre-training, we employ MoralBERT~\citep{trager2022moral}, a fine-tuned language model for capturing moral-values in social discussions\footnote{https://github.com/vjosapreniqi/MoralBERT}. Specifically, we use ten independently fine-tuned classifiers, each corresponding to one of ten moral foundations defined by MFT, i.e., $\mathcal{C}=\{$\textit{care}, \textit{harm}, \textit{fairness}, \textit{cheating}, \textit{loyalty}, \textit{betrayal}, \textit{authority}, \textit{subversion}, \textit{purity}, \textit{degradation}$\}$. 
Given a post $v$, each classifier produces a probability score $\text{MoralBERT}(v,c) \in [0,1]$ for each foundation $c\in\mathcal{C}$. 
To obtain user-level signal, we aggregate scores across posts authored by a user: 
\begin{equation}
   \text{score}(u,c)=\frac{1}{|\phi(u)|}\sum_{v\in\phi(u)}\text{MoralBERT}(v,c).
\end{equation}
The resulting 10-D vector is not treated as gold user values; it serves as a noisy but more robust moral-framing signal, whose usefulness comes from relative differences across users rather than exact post-level predictions.

\paragraph{Contrastive User Pairs Construction.}  
Given inferred value vectors derived from MFT, we sample user pairs $\{(u,\tilde{u}) \mid u, \tilde{u} \in \mathcal{U}~\mathrm{and}~u \neq \tilde{u}\}$ according to a similarity function $\mathrm{sim}(h(u),h(\tilde{u}))$. 
We adopt a symmetric percentile-based strategy, where the $P$-th percentile (e.g., $P=90\%$) defines positive pairs and the $(1-P)\%$-th percentile defines negative pairs:
\begin{equation}
\begin{aligned}
\theta_{+} &= \operatorname{ScoreAt}_{P}\bigl(\{\mathrm{sim}(h(u),h(\tilde{u}))\}\bigr),\\
\theta_{-} &= \operatorname{ScoreAt}_{(1-P)}\bigl(\{\mathrm{sim}(h(u),h(\tilde{u}))\}\bigr).
\end{aligned}
\end{equation}
For each user $u$, we define positive and negative sets as ${S}(u)$ and ${D}(u)$.
Assuming an approximately symmetric similarity distribution (see Appendix: Distribution of Sampled User Similarities), this strategy yields balanced positive and negative sets. If either set is underpopulated for a given user, we iteratively relax the threshold to the median similarity until a minimum set size is reached.

\paragraph{Similarity Computation.} 
User similarity is computed from the Euclidean distance between inferred value vectors:
\begin{equation}
    d(u,\tilde{u})=\|h(u)-h(\tilde{u})\|_2,
\end{equation}
which is converted into a similarity score through a Gaussian radial basis function kernel~\citep{li2021self}:
\begin{equation}
\mathrm{sim}(h(u),h(\tilde{u}))=
    \exp\left({-\frac{d(u,\tilde{u})^2}{2 \sigma^2}}\right),
\end{equation}
where $\sigma$ is set to the median of the sampled distances. 

\paragraph{Loss Functions.}
We treat the value vectors from text as a noisy auxiliary signal that defines relative constraints among users. This places a twofold requirement on our training objective: it should both fully exploit the inter-user relations encoded in value similarity and remain robust to the noise in the signal. Based on this, we design two complementary losses. The full training is illustrated in Appendix: Algorithm.

\textbf{(1) User-level contrastive loss $\mathcal{L}_{\text{user}}$:}
For each seed user $u \in U_{\text{seed}}$, we encourage its embedding to be close to value-similar $\tilde{u} \in S(u)$ and distant from value-dissimilar users $\hat{u} \in D(u)$. For each positive pair $(u, \tilde{u})$, we apply the InfoNCE loss:
\begin{equation}
    \ell_{u,\tilde{u}} = -\log
    \frac{e^{\cos(u,\tilde{u})/\tau}}
    {e^{\cos(u,\tilde{u})/\tau} + \sum_{\hat{u}\in D(u)} e^{\cos(u,\hat{u})/\tau}},
\end{equation}
where $\tau > 0$ is a temperature and $\cos(x,y)$ denotes cosine similarity. The overall contrastive loss is:
\begin{equation}
\label{eq:luser}
    \mathcal{L}_{\text{user}} = \frac{1}{N_{\text{pos}}}
    \sum_{u \in U_{\text{seed}}} \sum_{\tilde{u} \in S(u)} \ell_{u,\tilde{u}},
\end{equation}
where
\begin{equation}
    N_{\text{pos}} = \sum_{u \in U_{\text{seed}}} |S(u)|.
\end{equation}
This loss constructs positive and negative pairs according to the inferred value similarity and aligns user representations at the semantic level, pulling users with similar inferred value-related textual patterns together and pushing users with dissimilar values apart.

\textbf{(2) Clustering loss $\mathcal{L}_{\text{cls}}$:}
To regularize the global structure of the user embedding space, we periodically perform $K$-means clustering on the user embeddings ${z_u}$ every $X$ epochs. Such periodic cluster assignment follows the alternating optimization paradigm commonly adopted in deep clustering methods, where cluster assignments are periodically updated while the encoder progressively refines representations~\citep{Caron_2018_ECCV}. In our framework, clustering serves as a global embedding-space regularizer rather than the primary learning objective. 
Let $c_i$ denote the centroid of cluster $i$ and $\ell_u$ be the cluster assignment of user $u$. Following maximum-margin clustering, we first define a compactness term that encourages each embedding to stay close to its assigned centroid:
\begin{equation}
\label{eq:lcompact}
    \mathcal{L}_{\text{compact}} = \frac{1}{|U^{+}_{\text{seed}}|}
    \sum_{u} \big\lVert z_u - c_{\ell_u} \big\rVert_2^2 .
\end{equation}
We further define a margin term that enforces a minimum distance $m$ between distinct cluster centroids:
\begin{equation}
\label{eq:lmargin}
    \mathcal{L}_{\text{margin}} = \sum_{0 \le p < q \le K-1}
    \big[\max(0,\, m - \lVert c_p - c_q \rVert_2)\big]^2 .
\end{equation}
Combining the two, the clustering loss is defined as:
\begin{equation}
\label{eq:lcls}
    \mathcal{L}_{\text{cls}} = \mathcal{L}_{\text{compact}}
    + \beta\,\frac{2}{K(K-1)}\,\mathcal{L}_{\text{margin}},
\end{equation}
where $\beta$ weights the margin term. For epochs without clustering, $\mathcal{L}_{\text{cls}}$ is set to zero. This loss imposes a global geometric constraint on the user embeddings, with the compactness term encouraging intra-cluster cohesion and the margin term encouraging inter-cluster separation.

\textbf{(3) The overall loss $\mathcal{L}$:}
Each loss on its own covers only half of the training objective. $\mathcal{L}_{\text{user}}$ is a purely pairwise constraint that only specifies which users should be close to which, without constraining the global geometry of the embedding space. $\mathcal{L}_{\text{cls}}$, in contrast, is a self-reinforcing objective that clusters the current embeddings and pulls them toward centroids: it amplifies whatever structure already exists in the representations, without distinguishing whether that structure reflects genuine user differences. Consequently, using either loss in isolation could be unstable.
We therefore design the final objective as a combination of the two:
\begin{equation}
\label{eq:total}
    \mathcal{L} = \mathcal{L}_{\text{user}} + \lambda \mathcal{L}_{\text{cls}},
\end{equation}
where $\lambda$ controls the contribution of the clustering term. The two losses act as mutual regularizers, in which $\mathcal{L}_{\text{user}}$ provides $\mathcal{L}_{\text{cls}}$ with a meaningful clustering axis aligned along value-relevant dimensions, while $\mathcal{L}_{\text{cls}}$ provides $\mathcal{L}_{\text{user}}$ with global stabilization and denoising by aggregating large numbers of users into groups. Representations that simultaneously satisfy local (pairwise) value consistency and global (group) structural consistency are preserved, while noisy structure is filtered out. Appendix: Training Loss and Hyperparameter Tuning gives training loss curves and hyperparameter tuning details.

\section{Theoretical Analysis}

We provide a mechanism-level characterization of how the training objectives shape the embedding space, which formalizes how value-signal contrastive learning pulls users with similar inferred profiles closer, while clustering promotes compact and separated user groups. Full Proofs are provided in Appendix: Full Proof of Theorem 1 and 2.


\begin{theorem}[Value-Signal Similarity Preservation]\label{thm:value-similarity}
Let \( h(u) \in \mathbb{R}^{10} \) denote the inferred value vector of user \(u\), and let \( z_u \in \mathbb{R}^d \) be the user embedding learned by \textsc{ValueGraph}. Suppose \( \mathcal{L}_{\mathrm{user}} \) converges under fixed positive and negative sets constructed from $h(\cdot)$. For sampled user pairs, the learned embedding similarity is encouraged to preserve the ordering induced by value-signal similarity:
\[
\text{sim}\bigl(h(u),h(\tilde{u})\bigr) \uparrow
\quad \Rightarrow \quad
\text{sim}\bigl(z_u,z_{\tilde{u}}\bigr) \uparrow .
\]
\end{theorem}


\begin{theorem}[Cluster Compactness and Separation]\label{thm:cluster-separation}
Let user embeddings $\{z_u\}_{u \in \mathcal{U}}$ be optimized with \( \mathcal{L}_{\mathrm{cls}} \), and let \( \{c_k\}_{k=1}^K \) denote the resulting cluster centroids. The compactness term minimizes \( \| z_u - c_{\ell_u} \|_2 \) for users assigned to cluster \(\ell_u\), while the margin term penalizes centroid pairs with distance below \(m\) and therefore encourages inter-cluster separation.
\end{theorem}

Theorems~\ref{thm:value-similarity} and~\ref{thm:cluster-separation} clarify the optimization behavior: \textsc{ValueGraph} uses inferred value signals to organize user embeddings, while graph pre-training supplies semantic and relational context.

\section{Experiments and Results}\label{experiments}

\subsection{Pre-training Corpus}\label{sect:pre-training-data} 
We construct a large-scale pre-training corpus from social media conversations collected from Reddit and Twitter. 
\textbf{Reddit dataset}\footnote{\url{https://convokit.cornell.edu/documentation/subreddit.html}} provides diverse community discussions across subreddits, where we follow~\citet{trager2022moral} to select communities reflecting diverse moral concerns. 
\textbf{Twitter datasets} include widely used rumor detection benchmarks, including PHEME~\citep{pheme}, Twitter16~\citep{10.5555/3061053.3061153}, BEARD~\citep{zeng-gao-2022-early}, and Twitter-Covid~\citep{twitter_covid}, which contain conversations involving rumors, controversies, and value-related debates. 
These datasets provide rich textual and reply-based interaction signals for learning contextualized user representations. 
After preprocessing, the combined corpus consists of 461,198 conversation graphs with approximately 13.6 million nodes and 39.9 million edges.

\begin{table*}[t!]
\centering
\renewcommand{\arraystretch}{1.1}
\resizebox{\textwidth}{!}{%
\begin{tabular}{l l l c cc ccc ccc c}
\toprule
& & & \multicolumn{1}{c}{\textbf{LLM}} 
      & \multicolumn{2}{c}{\textbf{Graph Pre-training}} 
      & \multicolumn{3}{c}{\textbf{PLM}} 
      & \multicolumn{3}{c}{\textbf{GNN}} 
      & \multicolumn{1}{c}{\textbf{ValueGraph}} \\
\cmidrule(lr){4-4} \cmidrule(lr){5-6} \cmidrule(lr){7-9} \cmidrule(lr){10-12}
\textbf{Model} & \textbf{Dataset} & \textbf{Metric} 
& \textbf{GPT-5.4} 
& \textbf{GraphCL} & \textbf{GraphMAE2} 
& \textbf{SimCSE} & \textbf{ModernBERT} & \textbf{Bertweet} 
& \textbf{GCN} & \textbf{GAT} & \textbf{GTN}  
&  \\
\midrule
\multirow{2}{*}{GLAN} 
& \multirow{2}{*}{MT\_CSD} 
& Acc.    & 0.58 & 0.41& 0.60& 0.60& 0.56& 0.58& 0.40& 0.42& 0.57& \textbf{0.63}$^{*}$ \\
& & MacF1 & 0.56 & 0.38& 0.55& 0.56& 0.53& 0.51& 0.35& 0.33& 0.56& \textbf{0.58}$^{~}$ \\
\midrule
\multirow{2}{*}{BrLSTM} 
& \multirow{2}{*}{RumourEval19} 
& Acc.    & 0.73 & 0.59& 0.67& 0.74& 0.74& 0.65& 0.61& 0.31& 0.71& \textbf{0.77}$^{*}$ \\
& & MacF1 & 0.47 & 0.38& 0.47& 0.53& 0.59& 0.56& 0.25& 0.18& 0.45& \textbf{0.71}$^{*}$ \\
\bottomrule
\end{tabular}%
}
\caption{Stance detection results on MT\_CSD and RumourEval19.  $^{*}$ indicates statistical significance at 99\% confidence level compared with the second-best performance based on two-tailed paired Student's $t$-test.}
\label{results:stance detection}
\end{table*}

\subsection{Stance Detection}

Stance detection classifies users’ attitudes toward a target as \emph{support}, \emph{oppose}, or \emph{neutral}. Because socio-political stances are often associated with value signal and ethical orientation~\citep{alkhatib-2020-personal,durmus-2018-beliefs}, \textsc{ValueGraph} provides value signal that serve as a useful inductive bias for stance inference. 
We evaluate \textsc{ValueGraph} on two benchmarks, MT\_CSD~\citep{niu-etal-2024-challenge} and RumourEval19~\citep{gorrell-etal-2019-semeval}.
We integrate our post-level embeddings into two established stance models. Details are shown in Appendix: Stance Detection. 

\textbf{(1) On MT\_CSD}, we adopt the graph-prompt framework of~\citet{zhao2024graphprompt}, which performs event-aware prompt tuning over conversation graphs. We replace its GNN node embeddings with our post embeddings and use GLAN~\citep{yuan2019jointlyembeddinglocalglobal} for prediction.  

\textbf{(2) On RumourEval19}, we use BrLSTM~\citep{ijcai2022p533}, which models tree-structured reply dependencies, and replace its initial token-level inputs with precomputed post embeddings.  
This plug-in strategy provides contextual and relational representations without modifying downstream architectures or adding supervision. 

\paragraph{Baselines.} 
We compare against four categories of strong baseline encoders.  

\textbf{(1) Graph-based pre-training:} GraphCL~\citep{Hafidi2020GraphCLCS}, which applies contrastive learning over augmented graph views, and GraphMAE2~\citep{10.1145/3543507.3583379}, a masked graph autoencoder. Both are pre-trained on the same datasets using an identical graph encoder and subsequently fine-tuned for stance classification.  

\textbf{(2) Pre-trained language models:} BERTweet~\citep{nguyen-etal-2020-bertweet}, SimCSE~\citep{gao2021simcse}, and ModernBERT~\citep{warner2024smarterbetterfasterlonger}. All these pre-trained language models are fine-tuned end-to-end on the stance detection task.  

\textbf{(3) Inductive GNNs:} GCN~\citep{kipf2016semi}, GTN~\citep{yun2019graph}, and GAT~\citep{velickovic2017graph}, trained from scratch on the downstream task to test whether graph structure alone, without graph pre-training or value-signal guidance, is sufficient. 

\textbf{(4) Text-only LLM:} GPT-5.4, prompted with the target, available user posts, and thread text with no graph serialization.\footnote{GPT-5.4 does not receive explicit graph; structure-aware graph serialization for LLMs is nontrivial and left to future work.}

For fair comparison, trainable models use identical data splits, optimizer configurations, and early-stopping criteria based on validation loss. GPT-5.4 is evaluated on the same test splits with fixed prompts and deterministic decoding. Fixed hyperparameters and prompt templates are provided in Appendix: Hyperparameter Settings and GPT-5.4 Prompt Templates.



\paragraph{Evaluation Metrics.}  
We report accuracy (acc) and macro F1 (macF1). While accuracy provides an overall measure, macF1 is more informative under class imbalance, which is prevalent in stance detection datasets.

\paragraph{Results.}
As shown in Table~\ref{results:stance detection}, \textsc{ValueGraph} consistently outperforms the non-LLM and LLM baselines on both benchmarks; GPT-5.4 performs comparably with the non-LLM baselines. On MT\_CSD, \textsc{ValueGraph} attains 63\% acc and 58\% macF1, yielding relative gains of 5\% in acc and 5.5\% in macF1 over the strongest pre-trained baseline (GraphMAE2, 60\% acc and 55\% macF1). On RumourEval19, \textsc{ValueGraph} achieves 77\% acc and 71\% macF1, corresponding to a relative macF1 improvement of 20.3\% over the strongest PLM baseline (ModernBERT, 59\%) and 57.8\% over the best graph-based baseline (GTN, 45\%).

\subsection{Twitter Bot Detection}

\begin{table*}[t!]
  \centering
  \small
    \begin{tabular}{@{} l l c l l l l l @{}}
      \toprule
      \textbf{Method} & \textbf{Setting} & \textbf{Type} 
      & \textbf{Accuracy} & \textbf{MacF1} & \textbf{Precision} & \textbf{Recall} & \textbf{MCC} \\
      \midrule
      GPT-5.4 & Text-only & T & 56.9$(\pm0.1)$ & 52.1$(\pm0.1)$ & 57.1$(\pm0.1)$ & 54.9$(\pm0.1)$ & 11.8$(\pm0.1)$ \\
      \midrule
      \multirow{4}{*}{RoBERTa} 
      & Baseline & T 
      & 50.7$(\pm0.2)$ 
      & 54.8$(\pm0.5)$ 
      & 48.8$(\pm0.1)$ 
      & 62.5$(\pm1.2)$ 
      & – \\
      & \textsc{ValueGraph} & T 
      & 56.9$(\pm0.6)$$^{*}$$^{*}$  
      & 64.3$(\pm0.4)$$^{*}$$^{*}$  
      & 53.3$(\pm0.6)$$^{*}$$^{*}$   
      & 81.0$(\pm2.3)$$^{*}$  
      & – \\
      & Baseline & UT 
      & 63.1$(\pm0.5)$ 
      & 66.2$(\pm0.5)$ 
      & 58.9$(\pm0.5)$ 
      & 75.6$(\pm1.4)$ 
      & – \\
      & \textsc{ValueGraph} & UT 
      & 65.9$(\pm0.3)$$^{*}$$^{*}$   
      & 68.9$(\pm2.4)$$^{*}$  
      & 61.2$(\pm0.6)$ 
      & 78.7$(\pm1.8)$ $^{*}$ 
      & – \\
      \midrule
      \multirow{4}{*}{\shortstack{BotGCN\\(T5)}}  
      & Baseline & TG 
      & 56.8$(\pm6.0)$ 
      & 64.4$(\pm1.3)$ 
      & 53.7$(\pm5.1)$ 
      & 82.1$(\pm8.1)$ 
      & 17.9$(\pm9.7)$ \\
      & \textsc{ValueGraph} & TG 
      & 70.2$(\pm1.3)$$^{*}$$^{*}$   
      & 67.9$(\pm0.3)$$^{*}$$^{*}$    
      & 69.8$(\pm3.3)$$^{*}$    
      & 66.4$(\pm3.2)$$^{*}$$^{*}$    
      & 40.3$(\pm2.5)$$^{*}$$^{*}$    \\
      & Baseline & FTUG 
      & 70.5$(\pm0.6)$ 
      & 70.4$(\pm0.4)$ 
      & 67.1$(\pm1.2)$ 
      & 74.2$(\pm1.5)$ 
      & 41.3$(\pm1.1)$ \\
      & \textsc{ValueGraph} & FTUG 
      & 71.1$(\pm0.7)$$^{*}$$^{*}$ 
      & 71.1$(\pm0.3)$ 
      & 68.9$(\pm2.7)$$^{*}$$^{*}$ 
      & 72.2$(\pm5.0)$$^{*}$$^{*}$ 
      & 42.5$(\pm1.1)$\\
      \midrule
      \multirow{4}{*}{\shortstack{BotGAT\\(T5)}}  
      & Baseline & TG 
      & 47.4$(\pm0.0)$ 
      & 64.3$(\pm0.0)$
      & 47.5$(\pm0.0)$ 
      & 99.8$(\pm0.2)$ 
      & –0.6$(\pm0.9)$ \\
      & \textsc{ValueGraph} & TG 
      & 73.1$(\pm0.3)$$^{*}$$^{*}$   
      & 66.0$(\pm0.6)$$^{*}$$^{*}$    
      & 82.5$(\pm2.2)$$^{*}$$^{*}$    
      & 55.1$(\pm1.8)$$^{*}$$^{*}$    
      & 47.8$(\pm1.2)$$^{*}$$^{*}$    \\
      & Baseline & FTUG 
      & 73.2$(\pm2.6)$
      & 70.9$(\pm2.0)$
      & 74.6$(\pm8.2)$ 
      & 69.3$(\pm10.0)$ 
      & 47.3$(\pm4.7)$ \\
      & \textsc{ValueGraph} & FTUG 
      & 74.7$(\pm0.8)$$^{*}$$^{*}$ 
      & 73.0$(\pm0.3)$$^{*}$
      & 74.2$(\pm3.0)$ 
      & 72.1$(\pm3.3)$ 
      & 49.3$(\pm1.6)$$^{*}$$^{*}$ \\
      \midrule
      \multirow{4}{*}{\shortstack{BotRGCN\\(T5)}} 
      & Baseline & TG 
      & 69.2$(\pm1.5)$ 
      & 69.1$(\pm0.5)$ 
      & 66.2$(\pm3.4)$ 
      & 72.6$(\pm4.3)$ 
      & 38.9$(\pm2.4)$ \\
      & \textsc{ValueGraph} & TG 
      & 70.3$(\pm1.1)$ 
      & 69.1$(\pm0.5)$$^{*}$$^{*}$    
      & 68.5$(\pm2.9)$$^{*}$  
      & 70.0$(\pm3.5)$$^{*}$$^{*}$    
      & 40.7$(\pm2.0)$ \\
      & Baseline & FTUG 
      & 71.0$(\pm0.8)$ 
      & 73.5$(\pm0.5)$ 
      & 65.0$(\pm1.1)$ 
      & 84.8$(\pm1.7)$ 
      & 44.7$(\pm1.2)$ \\
      & \textsc{ValueGraph} & FTUG 
      & 73.0$(\pm0.8)$$^{*}$
      & 74.2$(\pm0.2)$ 
      & 68.0$(\pm1.7)$$^{*}$  
      & 82.0$(\pm2.2)$$^{*}$  
      & 47.0$(\pm1.1)$ \\
      \bottomrule
    \end{tabular}
    \caption{Comparison of text-only GPT-5.4, baseline embeddings, and \textsc{ValueGraph} embeddings for twitter bot detection. * and ** indicate statistical significance at 95\% and 99\% confidence levels, respectively, based on two-tailed paired Student’s $t$-test between \textsc{ValueGraph} and the corresponding baseline embedding.}
    \label{tab:twitbotmetrics}
\end{table*}

Detecting twitter bots is critical for maintaining the integrity of online discourse. However, most existing detectors rely primarily on surface-level textual features or network structure. We evaluate whether replacing or augmenting such cues with \textsc{ValueGraph} user embeddings guided by inferred value signals can improve bot detection performance across strong baseline models. We use TwiBot-22~\citep{10.5555/3600270.3602825} dataset for evaluation (Dataset details are in Appendix: Twitter Bot Detection).


\paragraph{Baselines.} We select high-performing models from the benchmark study by \citet{10.5555/3600270.3602825}, covering different data modalities.
We include RoBERTa~\citep{liu2019robertarobustlyoptimizedbert} as a text-based baseline, together with BotRGCN~\citep{feng2021botrgcn}, BotGAT~\citep{lei2022bic}, and BotGCN~\citep{feng2021botrgcn}, which use T5 encoder~\citep{T5} to encode text. 
We also add a \textbf{GPT-5.4} baseline using sampled user posts and profile text with no graph serialization. Each non-LLM model is evaluated under configurations using F (user metadata or engineered features), T (posts), U (profile description), and G (network structure). We examine whether substituting T/G-based representations with \textsc{ValueGraph} embeddings improves performance. 

\paragraph{Setup and Metrics.} 
All trainable encoders retain their original architectures and default settings. We perform grid search over hyperparameters (see  Appendix: Hyperparameter Settings), run each experiment five times with different random seeds, and report Accuracy, MacF1, Precision, Recall, and Matthews Correlation Coefficient (MCC).

\paragraph{Results.} 

Table~\ref{tab:twitbotmetrics} shows that \textsc{ValueGraph}-based embeddings consistently outperform non-LLM baseline representations across models. For instance, RoBERTa in the text-only configuration improves macF1 from 54.8\% to 64.3\%, with recall increasing from 62.5\% to 81.0\%. While GPT-5.4 achieves acc generally on par with \textsc{ValueGraph} in the text-only setting, its macF1 is substantially lower. These gains indicate that \textsc{ValueGraph} captures complementary semantic and structural cues for bot-human discrimination.
While the improvements are statistically significant in text+network settings, significance tends to decrease when full multimodal inputs are available, particularly for macF1 in FTUG configurations, where strong profile and graph features already dominate the decision boundary. Nevertheless, because the \textsc{ValueGraph} encoder is not trained with bot labels or downstream fine-tuning, these gains indicate that value-signal guidance adds useful information even in feature-rich settings.


\begin{figure*}[t!]
    \centering

    \begin{subfigure}[t]{0.24\textwidth}
        \centering
        \includegraphics[width=\linewidth]{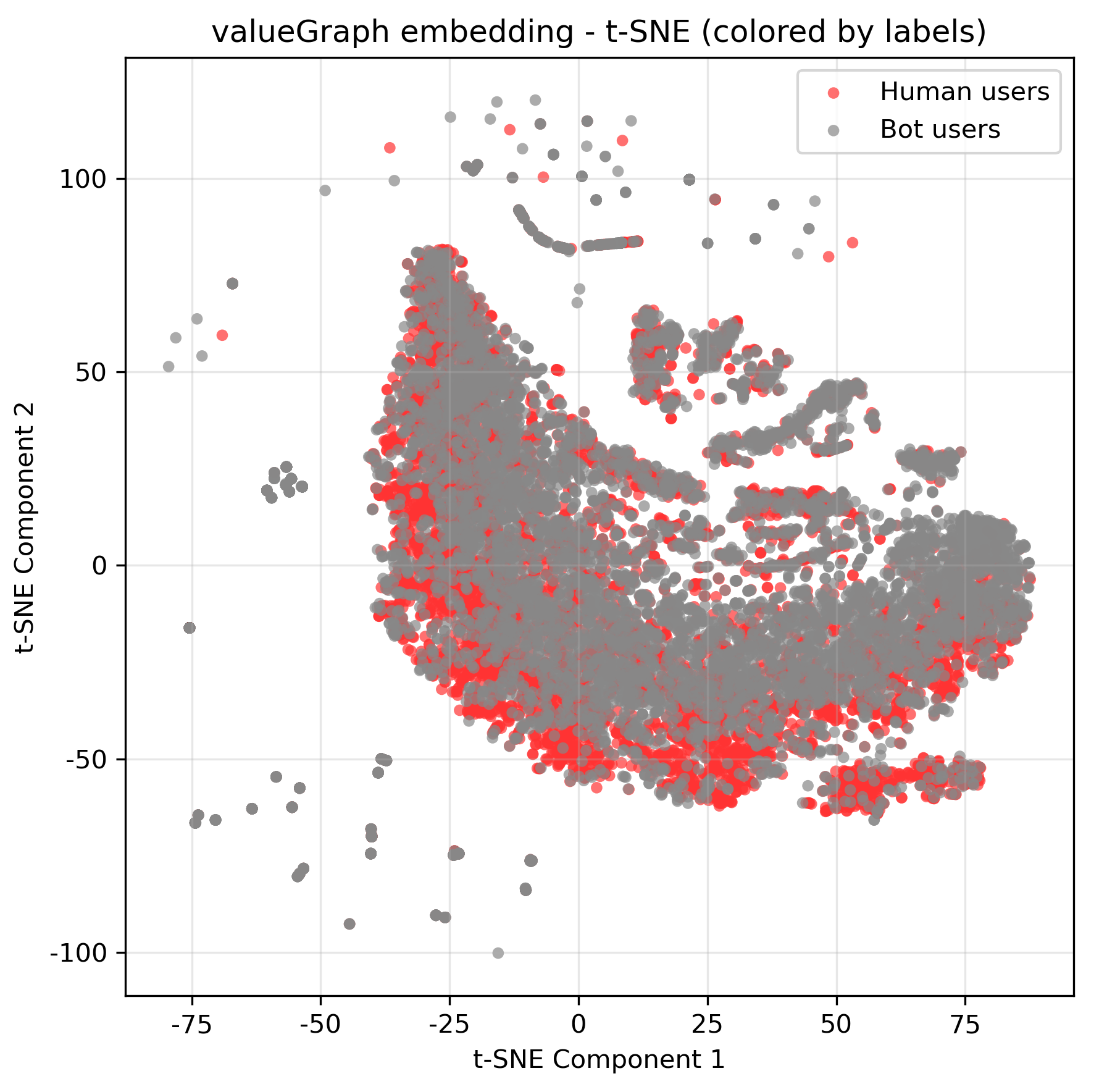}
        \caption{ValueGraph Embeddings}
        \label{fig:tsne-a}
    \end{subfigure}\hfill
    \begin{subfigure}[t]{0.23\textwidth}
        \centering
        \includegraphics[width=\linewidth]{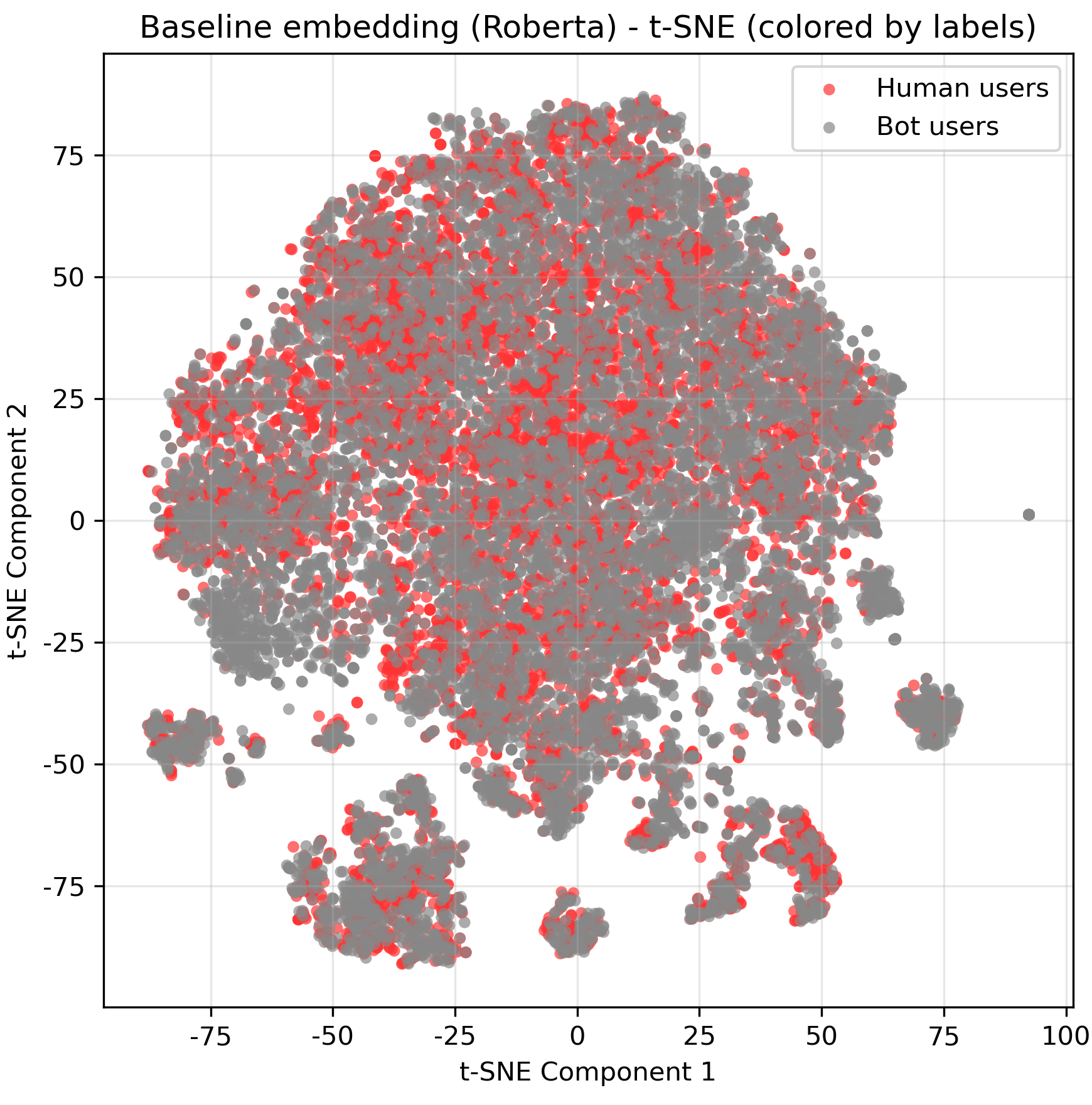}
        \caption{RoBERTa Embeddings}
        \label{fig:tsne-b}
    \end{subfigure}\hfill
    \begin{subfigure}[t]{0.23\textwidth}
        \centering
        \includegraphics[width=\linewidth]{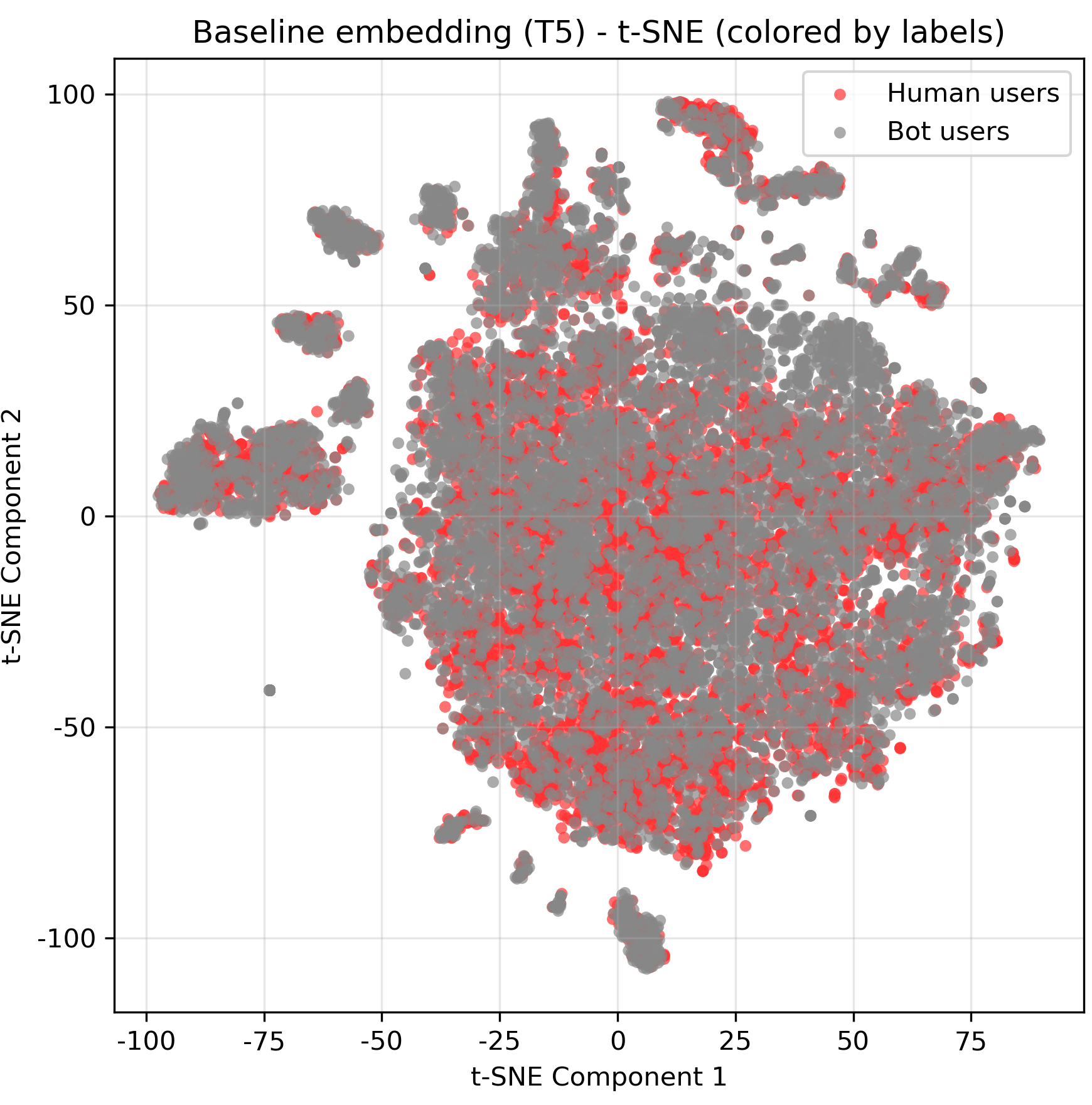}
        \caption{T5 Embeddings}
        \label{fig:tsne-c}
    \end{subfigure}\hfill
    \begin{subfigure}[t]{0.23\textwidth}
        \centering
        \includegraphics[width=\linewidth]{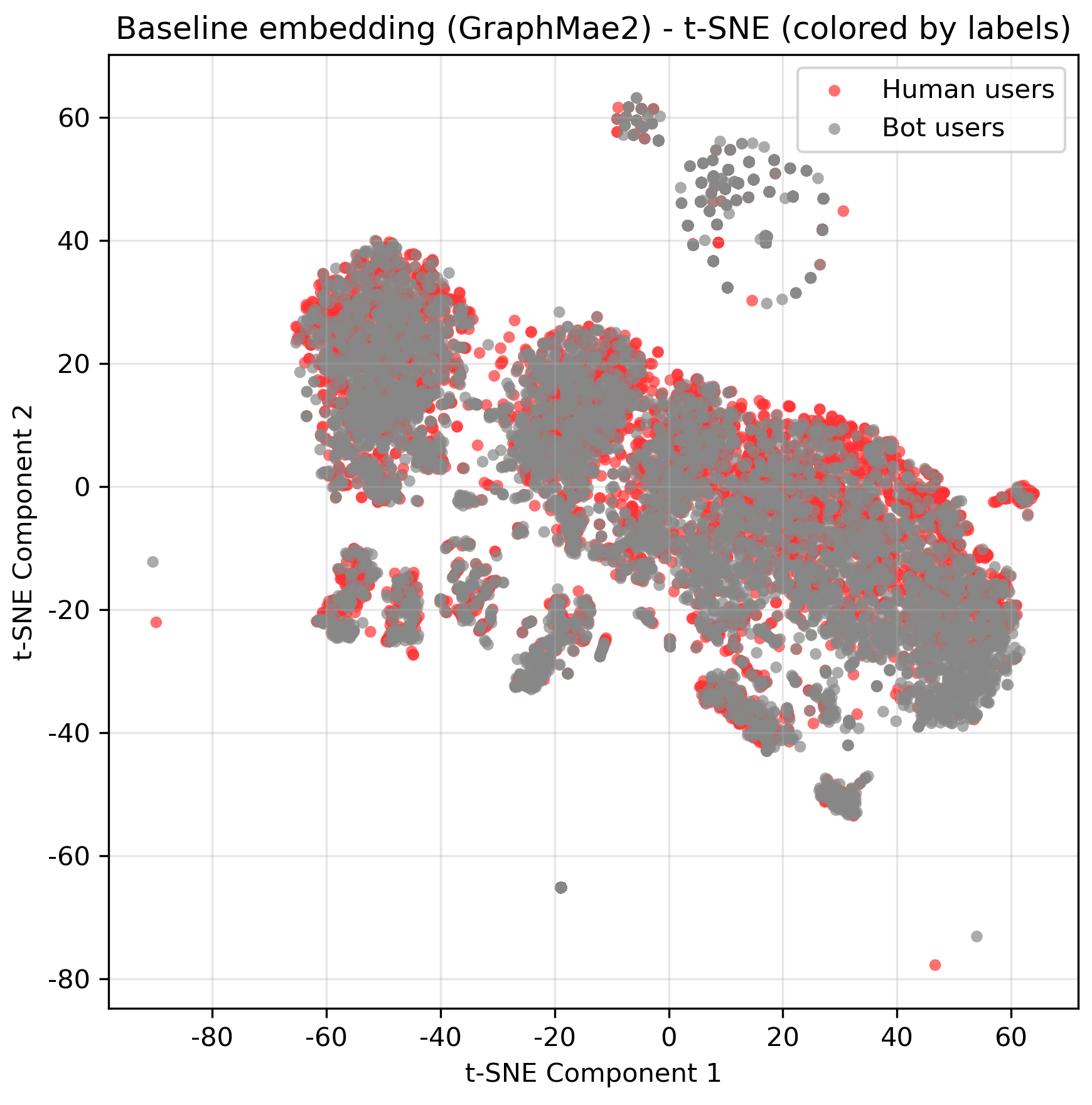}
        \caption{GraphMAE2 Embeddings}
        \label{fig:tsne-d}
    \end{subfigure}

    \caption{t-SNE visualization of user embeddings for Twitter bot detection. Red and grey dots represent ground-truth human and bot users, respectively.}
    \label{fig:tsne_vis}
\end{figure*}

\subsection{Ablation Study}

\begin{table}[t!]
\small
\centering
\resizebox{0.48\textwidth}{!}{%
\begin{tabular}{l l l c c}
\toprule
\textbf{Task} & \textbf{Model/data} & \textbf{Configuration} &
\textbf{Accuracy} & \textbf{MacF1} \\
\midrule

\multirow{8}{*}{Stance}
  & \multirow{4}{*}{\shortstack{GLAN\\(MT\_CSD)}}
    & Only Stage 1              & 41.2 & 40.7 \\
  & & Stage 1 + $\mathcal{L}_{\text{cls}}$  & 50.4 & 34.3 \\
  & & Stage 1 + $\mathcal{L}_{\text{user}}$ & 45.3 & 32.6 \\
  & & \textbf{Full}             & \textbf{63.0} & \textbf{58.0} \\

\cdashline{2-5}

  & \multirow{4}{*}{\shortstack{BrLSTM\\(RumourEval19)}}
    & Only Stage 1              & 53.7 & 55.8 \\
  & & Stage 1 + $\mathcal{L}_{\text{cls}}$  & 62.7 & 52.8 \\
  & & Stage 1 + $\mathcal{L}_{\text{user}}$ & 53.9 & 45.7 \\
  & & \textbf{Full}             & \textbf{77.1} & \textbf{71.2} \\

\midrule



\multirow{4}{*}{Bot}
  & \multirow{4}{*}{\shortstack{BotGAT\\(TwiBot-22)}}
    & Only Stage 1              & 73.1 & 70.3 \\
  & & Stage 1 + $\mathcal{L}_{\text{cls}}$  & 72.9 & 72.6 \\
  & & Stage 1 + $\mathcal{L}_{\text{user}}$ & 72.6 & 72.0 \\
  & & \textbf{Full}             & \textbf{74.7} & \textbf{73.0} \\
\bottomrule
\end{tabular}
}
\caption{Ablation results of \textsc{ValueGraph} embeddings under different training settings across stance detection and Twitter bot detection tasks. 
For Twitter bot detection, BotGAT is evaluated in the FTUG setting.}
\label{tab:ablation}
\end{table}
To assess the contribution of value-signal information, we conduct controlled
ablation experiments. Since Stage~1 pre-training is essential for capturing reply
structure and producing stable representations, it is retained in all settings.
We evaluate four variants: (1) Stage~1 pre-training only; (2) Stage~1 $+$
$\mathcal{L}_{\text{cls}}$; (3) Stage~1 $+$ $\mathcal{L}_{\text{user}}$; and
(4) the full model.

Table~\ref{tab:ablation} shows that neither $\mathcal{L}_{\text{cls}}$ nor $\mathcal{L}_{\text{user}}$ alone yields consistent improvements, which supports our conjecture in the Loss Function section. Notably, a single loss can raise accuracy while lowering macF1. For example, on GLAN (MT\_CSD), adding $\mathcal{L}_{\text{cls}}$ improves accuracy from $41.2\%$ to
$50.4\%$ but drops macF1 from $40.7\%$ to $34.3\%$. 
In contrast, the full \textsc{ValueGraph} model, which jointly integrates
structural pre-training, clustering supervision, and value-signal guidance,
achieves consistent improvements across all tasks. These results validate the
complete \textsc{ValueGraph} design (see Appendix: Supplementary Results of the Ablation Experiments for full results).


\subsection{User Clustering Analysis}
We use t-SNE (see Appendix: t-SNE Setting for parameter setup) to visualize how different embedding models organize users for Twitter bot detection. As shown in Figure~\ref{fig:tsne_vis}, \textsc{ValueGraph} yields clearer bot-human separation than the compared text and graph encoders. This visualization provides dataset-specific qualitative evidence of improved representation separability.

A content-level comparison helps explain this separation. In the evaluated dataset, some bot clusters contain repetitive and polarized value signal. For example, the bot account (author\_id anonymized) posted an accusatory narrative: ``\textit{Paul Vickers died suddenly 3 months after I exposed his \#Establishment lies … contributed to his death … \#FactCheck …}'', relying on blame attribution, repeated hashtags, and a one-sided moral stance. In contrast, a human account (author\_id anonymized) wrote: ``\textit{Is that what the point is? Being angry with the one we love is not the same as unloving them, is it?}'', showing more dialogic and context-sensitive expression. These examples suggest that \textsc{ValueGraph} captures value signals useful for bot detection.

\section{Conclusion}
We presented \textsc{ValueGraph}, a value-signal guided graph pre-training framework for contextualized user representation. Grounded in MFT, \textsc{ValueGraph} combines semantic and structural graph pre-training, contrastive learning over inferred value similarity, and clustering to capture textual, relational, and value-relevant behavioral cues. Experiments on stance detection and Twitter bot detection show consistent gains over strong baselines. These results suggest that noisy but theory-informed value signals provide useful auxiliary guidance for socially informed user modeling without predicting users' true values.

\section*{Ethics Statement}

This work uses publicly available social-media datasets under their original access conditions and licenses. We do not infer users' true psychological values; MoralBERT-derived vectors are used only as noisy aggregate signals for representation learning. Results should be interpreted as dataset-specific and should not be used for individual-level profiling or decision-making without appropriate safeguards.
We will release code and processing scripts but will not redistribute raw social-media content or personally identifiable information. The released artifacts are intended for research purposes only.

\section*{Acknowledgment}
This research is supported by the SMU-A*STAR Joint Lab in Social and Human-Centered Computing (SMU grant no.: SAJL-2022-CSS02, SAJL-2022-CSS003). This research is supported by A*STAR (C232918004, C232918005).

\bibliography{aaai2027}


\clearpage
\input{appendix.tex}

\end{document}

%% file: appendix.tex
\appendix

\section*{Appendix}\label{appendix}

\section{Full Proof of Theorem 1}\label{app:proof-value-similarity}

\begin{proof}
The gradient of the InfoNCE loss with respect to the embedding $z_u$ takes the form
\begin{equation}
\nabla_{z_u} \ell_{u,\tilde{u}} \propto \frac{1}{\tau}\Bigl(z_{\tilde{u}}-\sum_{\hat{u}\in \mathcal{D}(u)} p_{\hat{u}} z_{\hat{u}}\Bigr),
\end{equation}
where, for simplicity, all embeddings are assumed to be normalized so that cosine similarity can be treated as a dot product. Define
\begin{equation}
\begin{aligned}
A &= e^{\cos(z_u,z_{\tilde{u}})/\tau}, \\
B &= e^{\cos(z_u,z_{\tilde{u}})/\tau} + \sum_{\hat{u}\in \mathcal{D}(u)} e^{\cos(z_u,z_{\hat{u}})/\tau}.
\end{aligned}
\end{equation}
Then the loss can be written as
\begin{equation}
\begin{aligned}
\ell_{u,\tilde{u}} &= -\log \frac{A}{B} = -\log A + \log B \\
&= -\frac{\cos(z_u,z_{\tilde{u}})}{\tau} + \log B.
\end{aligned}
\end{equation}
Differentiating the first term gives
\begin{equation}
\nabla_{z_u}\left[-\frac{\cos(z_u,z_{\tilde{u}})}{\tau} \right] = -\frac{1}{\tau} z_{\tilde{u}}.
\end{equation}
Differentiating the second term gives
\begin{equation}
\nabla_{z_u}B = \frac{1}{\tau}e^{\cos(z_u,z_{\tilde{u}})/\tau}z_{\tilde{u}} + \sum_{\hat{u}\in\mathcal{D}(u)}\frac{1}{\tau}e^{\cos(z_u,z_{\hat{u}})/\tau}z_{\hat{u}}.
\end{equation}
Define the softmax probabilities as
\begin{equation}
\begin{aligned}
p_{\tilde{u}} &= \frac{e^{\cos(z_u,z_{\tilde{u}})/\tau}}{B}, \\
p_{\hat{u}} &= \frac{e^{\cos(z_u,z_{\hat{u}})/\tau}}{B}, \quad \hat{u} \in \mathcal{D}(u).
\end{aligned}
\end{equation}
The gradient of the second term is therefore
\begin{equation}
\nabla_{z_u}\log B = \frac{1}{\tau}\Bigl( p_{\tilde{u}}z_{\tilde{u}} + \sum_{\hat{u}\in\mathcal{D}(u)} p_{\hat{u}}z_{\hat{u}}\Bigr),
\end{equation}
and the full gradient becomes
\begin{equation}
\nabla_{z_u}\ell_{u,\tilde{u}} = \frac{1}{\tau}\Bigl((p_{\tilde{u}}-1)z_{\tilde{u}} + \sum_{\hat{u}\in\mathcal{D}(u)} p_{\hat{u}}z_{\hat{u}}\Bigr).
\end{equation}

By construction, positive pairs have high inferred value-signal similarity, while negative pairs have low inferred value-signal similarity. The InfoNCE objective therefore pulls $z_u$ toward embeddings of value-similar users and pushes it away from value-dissimilar users. Under gradient descent, the update for $z_u$ is
\begin{equation}
z_u^{(t+1)} = z_u^{(t)} + \frac{\eta}{\tau}\Bigl(z_{\tilde{u}} - \sum_{\hat{u}\in\mathcal{D}(u)}p_{\hat{u}}z_{\hat{u}}\Bigr).
\end{equation}
When $h(u)$ and $h(\tilde{u})$ are close, the positive pair dominates the softmax, so the update primarily pulls $z_u$ toward $z_{\tilde{u}}$. This yields
\begin{equation}
\cos\bigl(z_u,z_{\tilde{u}}\bigr) \uparrow \quad \text{when} \quad \cos\bigl(h(u),h(\tilde{u})\bigr) \uparrow.
\end{equation}
Thus, the learned embedding similarity is encouraged to preserve the ordering induced by inferred value-signal similarity.
\end{proof}

\section{Full Proof of Theorem 2}\label{app:proof-cluster-separation}

\begin{proof}
The compactness loss
\begin{equation}
\mathcal{L}_{\text{compact}} = \frac{1}{|\mathcal{U}|} \sum_{u\in \mathcal{U}} \| z_u - c_{\ell_u}\|_2^2
\end{equation}
minimizes the squared Euclidean distance between each user embedding $z_u$ and its assigned centroid $c_{\ell_u}$. If $z_u$ deviates from $c_{\ell_u}$, the loss increases; optimization therefore pulls embeddings toward their assigned centroids and encourages intra-cluster compactness.

The margin loss is
\begin{equation}
\mathcal{L}_{\text{margin}} = \sum_{p<q}\left[ \max\left(0, m - \| c_p - c_q \|_2 \right) \right]^2.
\end{equation}
For any pair of distinct centroids $c_p$ and $c_q$, if $\| c_p - c_q \|_2 \ge m$, the corresponding loss term is zero and no further separation force is applied. If $\| c_p - c_q \|_2 < m$, the pair contributes
\begin{equation}
\ell_{pq} = \left( m - \| c_p - c_q \|_2 \right)^2.
\end{equation}
Let $d=\|c_p-c_q\|_2$. The gradient with respect to $c_p$ is
\begin{equation}
\nabla_{c_p} \ell_{pq} = -2\left( m - d\right) \frac{c_p - c_q}{d},
\end{equation}
with a symmetric expression for $c_q$. Thus, when $d<m$, the gradient pushes the centroids apart until their distance reaches the margin.

The total loss is
\begin{equation}
\mathcal{L} = \mathcal{L}_{\text{user}} + \lambda \mathcal{L}_{\text{cls}},
\end{equation}
where $\mathcal{L}_{\text{cls}}$ includes both compactness and margin terms. Therefore, minimizing $\mathcal{L}_{\text{cls}}$ jointly encourages embeddings within the same cluster to remain close while pushing different cluster centroids apart. This establishes the intended compactness and separation properties.
\end{proof}

\section{Moral Foundations Background}\label{appendix:mft}

\textbf{Moral Foundations Theory (MFT)}: To assess individuals' moral foundations, \citet{haidt_all} developed survey-based questions using factor analysis. We summarize the core moral foundations below:

\begin{itemize}
  \item \textbf{Care/Harm}: This foundation reflects the tendency to form emotional bonds and experience distress at others' suffering. It emphasizes kindness, compassion, and nurturance, while discouraging harm.
  \item \textbf{Fairness/Cheating}: Rooted in reciprocal altruism, this foundation promotes justice, equity, proportionality, and autonomy.
  \item \textbf{Loyalty/Betrayal}: This foundation reflects loyalty to one's in-group or community. It supports patriotism and self-sacrifice, but in extreme cases can lead to nepotism or favoritism.
  \item \textbf{Authority/Subversion}: This foundation reflects respect for leadership, hierarchy, tradition, and established social structures.
  \item \textbf{Sanctity/Degradation}: This foundation reflects the inclination to uphold purity and avoid contamination. It is often tied to religious and cultural ideals of moral elevation.
\end{itemize}

\section{Distribution of Sampled User Similarities}\label{app:distribution}
\begin{figure}
    \centering
    \includegraphics[width=0.9\linewidth]{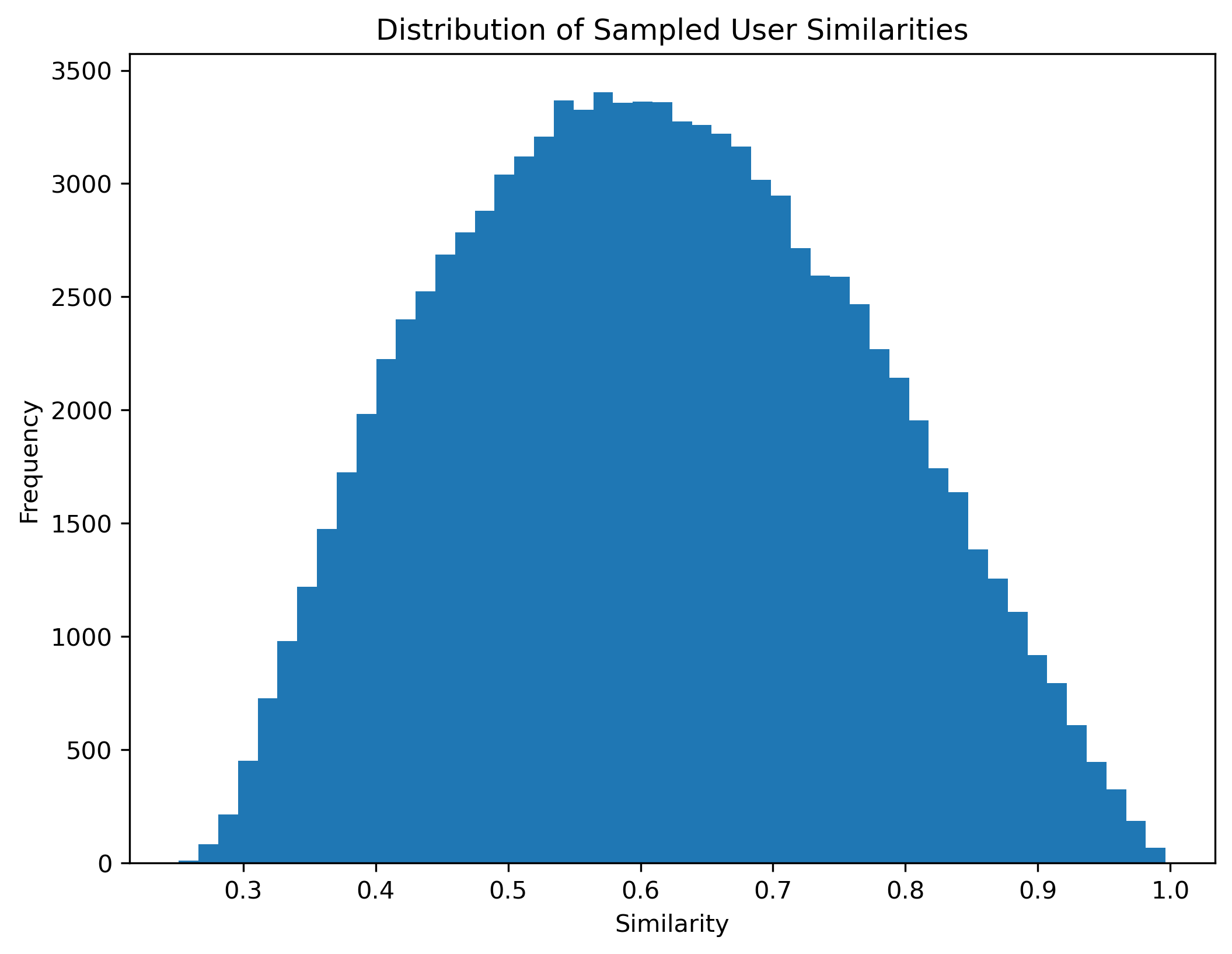}
    \caption{Distribution of RBF-based similarities for 100,000 randomly sampled user pairs.}
    \label{fig:Distribution}
\end{figure}

As shown in Figure~\ref{fig:Distribution}, similarity scores are concentrated at lower values, with a long tail toward higher similarities. This indicates that most user pairs in the inferred value-profile space have weak affinity, while a small subset shows strong alignment.

Based on this distribution, we use percentile-based sampling, selecting positive pairs from the 99.5-th percentile and negative pairs from the 0.5-th percentile. This strategy yields discriminative contrastive examples, reduces ambiguity from moderately similar pairs, and preserves enough samples for stable training.

\section{Training Loss and Hyperparameter Tuning}\label{ValueGraph training part}

To optimize \textsc{ValueGraph}, we conduct a grid search over key hyperparameters in the loss formulation. The objective combines the user-level contrastive loss $\mathcal{L}_{\mathrm{user}}$ and the clustering loss $\mathcal{L}_{\mathrm{cls}}$. Final hyperparameters are selected based on convergence behavior and downstream validation performance.
The tuned hyperparameters are:
\begin{itemize}
\item \textbf{Temperature $\tau$}: Controls the sharpness of similarity distributions in the InfoNCE loss. Smaller $\tau$ emphasizes hard negatives, while larger $\tau$ produces smoother gradients and more stable optimization.

\item \textbf{Margin coefficient $\beta$}: Scales the margin term $\mathcal{L}_{\text{margin}}$ in the clustering loss. Larger $\beta$ enforces stronger inter-cluster separation, while smaller $\beta$ allows more flexible cluster boundaries.

\item \textbf{Number of augmented users $k$}: The number of users sampled per seed user during contrastive training. Larger $k$ increases training diversity but also computational cost.

\item \textbf{Number of clusters $K$}: Controls the granularity of user grouping in the embedding space.

\item \textbf{Margin $m$}: Specifies the minimum distance encouraged between cluster centroids.
\end{itemize}

We report convergence plots for representative configurations in Figure~\ref{fig:valuegraph_loss}. The final configuration is: temperature $\tau=0.07$, margin coefficient $\beta=0.5$, augmented users $k=10$, clusters $K=5$, clustering weight $\lambda=0.05$, and margin $m=1$.

\begin{figure*}[htb!]
    \centering

    \begin{subfigure}{0.32\linewidth}
        \centering
        \includegraphics[width=\linewidth]{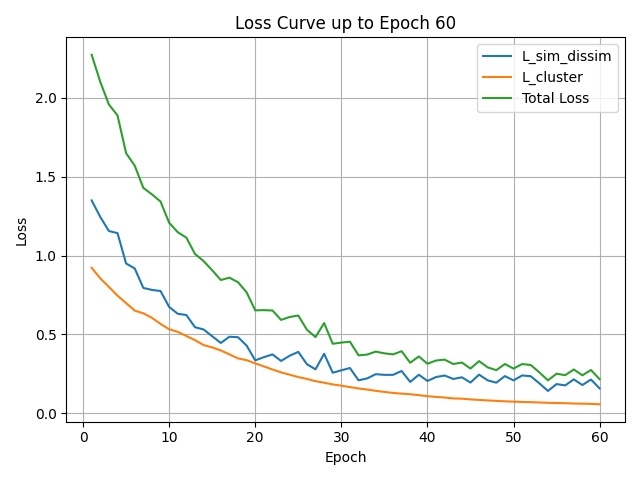}
        \caption{$K=20$, $k=10$}
        \label{fig:loss_1}
    \end{subfigure}
    \hfill
    \begin{subfigure}{0.32\linewidth}
        \centering
        \includegraphics[width=\linewidth]{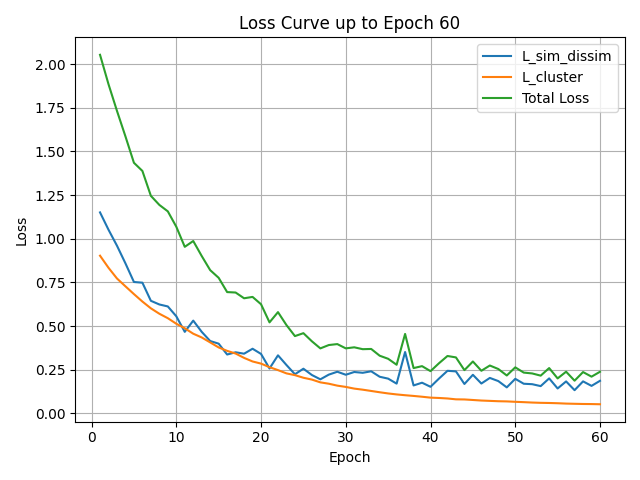}
        \caption{$K=10$, $k=10$}
        \label{fig:loss_2}
    \end{subfigure}
    \hfill
    \begin{subfigure}{0.32\linewidth}
        \centering
        \includegraphics[width=\linewidth]{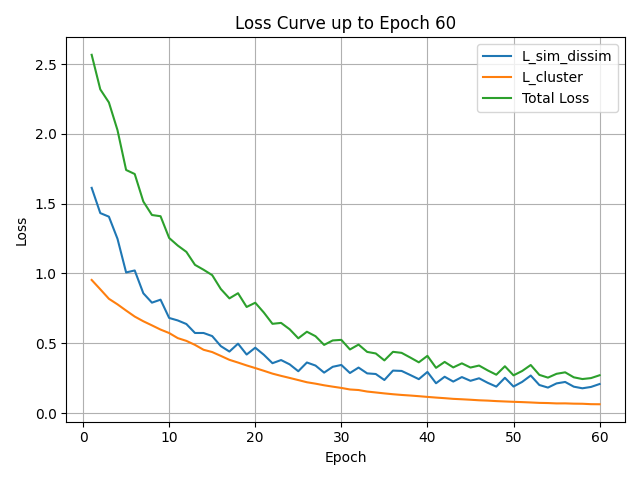}
        \caption{$K=5$, $k=20$}
        \label{fig:loss_3}
    \end{subfigure}

    \caption{Hyperparameter configurations used in \textsc{ValueGraph} training. All settings use temperature $\tau = 0.07$, margin coefficient $\beta = 0.5$, margin $m = 1$, and clustering loss weight $\lambda_{\mathrm{cluster}} = 0.05$. The three configurations vary in the number of augmented users $k$ and the number of clusters $K$.}
    \label{fig:valuegraph_loss}
\end{figure*}

\section{Evaluation Tasks and Datasets}\label{appendix:Downstream Tasks and Evaluation Datasets}

\subsection{Details of Pre-training Datasets}

We provide additional statistics of the pre-training datasets, including source domains, topic distributions, and graph statistics.

\paragraph{Reddit Corpus.}
Our Reddit corpus is constructed from the Reddit Corpus provided by ConvoKit\footnote{\url{https://convokit.cornell.edu/documentation/subreddit.html}}, following the dataset construction procedure of~\citet{trager2022moral}. It contains conversations collected from multiple subreddits representing diverse social communities and discussion topics. The selected subreddits cover political discussions, interpersonal relationships, social issues, and general-interest communities. We retain conversation threads with at least 50 comments to ensure sufficient interaction structures.

After preprocessing, including removing inaccessible comments, filtering isolated nodes, and preserving reply-based interactions, the Reddit corpus contains approximately 9.6 million comments. The detailed distribution across subreddits is shown in Table~\ref{tab:reddit_statistics}.

\begin{table}[h]
\centering
\begin{tabular}{l|r}
\hline
Subreddit & Comments \\
\hline
Politics & 5,626,191 \\
Neoliberal & 2,315,326 \\
Relationship Advice & 861,704 \\
Confession & 477,353 \\
Conservative & 165,182 \\
Worldnews & 100,194 \\
AmItheAsshole & 83,624 \\
Geopolitics & 41,700 \\
Antiwork & 452 \\
Nostalgia & 378 \\
\hline
Total & 9,672,104 \\
\hline
\end{tabular}
\caption{Statistics of the Reddit pre-training corpus.}
\label{tab:reddit_statistics}
\end{table}

\paragraph{Twitter Corpus.}
Our Twitter corpus combines several publicly available social media conversation datasets, including PHEME~\citep{pheme}, Twitter16~\citep{10.5555/3061053.3061153}, BEARD~\citep{zeng-gao-2022-early}, and Twitter-Covid~\citep{twitter_covid}. These datasets cover diverse scenarios such as rumor propagation, public health discussions, and event-driven social interactions.

During preprocessing, we remove isolated nodes without parent or child interactions and preserve only users participating in conversational structures. The final Twitter corpus contains over 5 million posts. The detailed statistics of each source dataset are summarized in Table~\ref{tab:twitter_statistics}.

\begin{table}[h]
\centering

\begin{tabular}{l|r}
\hline
Dataset & Posts \\
\hline
BEARD~\citep{zeng-gao-2022-early} & 3,393,578 \\
Twitter-Covid~\citep{twitter_covid} & 474,710 \\
PHEME~\citep{pheme} & 59,387 \\
Twitter16~\citep{10.5555/3061053.3061153} & 1,101,985 \\
\hline
Total & 5,029,660 \\
\hline
\end{tabular}
\caption{Statistics of the Twitter pre-training corpus.}
\label{tab:twitter_statistics}
\end{table}

\subsection{Stance Detection}\label{appendix:Stance exp-settings}
\textbf{MT\_CSD}~\citep{niu-etal-2024-challenge} is a benchmark for Multi-Target Conversational Stance Detection, constructed from authentic Reddit interactions. It contains 15,876 human-annotated instances over extended conversation threads, with approximately 76\% involving more than three reply turns. The dataset includes Reddit posts, popularity metrics, and discussions centered on targets such as Tesla, SpaceX, Donald Trump, Joe Biden, and Bitcoin.

\textbf{RumourEval19}~\citep{gorrell-etal-2019-semeval} Task A is a Twitter benchmark for rumor stance classification in conversational threads. It contains 15,874 tweets and provides tweet IDs and reply-to IDs, enabling reconstruction of tree- or graph-structured conversations.

Dataset statistics are shown in Tables~\ref{tab:dataset_stats} and~\ref{tab:train_dev_stats}.

\begin{table}[hbp!]
\centering

\begin{tabular}{llr}
\toprule
\textbf{Split} & \textbf{Label} & \textbf{Count} \\
\midrule
Train & Against & 3,349 \\
      & Favor   & 2,383 \\
      & None    & 5,441 \\
\midrule
Valid & Against & 710  \\
      & Favor   & 481  \\
      & None    & 1,115 \\
\midrule
Test  & Against & 728  \\
      & Favor   & 470  \\
      & None    & 1,197 \\
\bottomrule
\end{tabular}
\caption{Statistics for  MT\_CSD dataset}
\label{tab:dataset_stats}
\end{table}

\begin{table}[hbp!]
\centering

\begin{tabular}{llr}
\toprule
\textbf{Split} & \textbf{Label} & \textbf{Count} \\
\midrule
Train & Comment & 2,734 \\
      & Deny    & 333  \\
      & Query   & 330  \\
      & Support & 841  \\
\midrule
Dev   & Comment & 173  \\
      & Deny    & 11   \\
      & Query   & 28   \\
      & Support & 69   \\
\bottomrule
\end{tabular}
\caption{Statistics for RumourEval19}
\label{tab:train_dev_stats}
\end{table}

\subsection{Twitter Bot Detection}\label{appendix:TwitterBotexp-settings}

\textbf{TwiBot-22}~\citep{10.5555/3600270.3602825} is a graph-based bot detection benchmark and one of the largest annotated Twitter account datasets. It models a heterogeneous graph with user, tweet, and hashtag nodes connected by multiple edge types. We select representative baselines from TwiBot-22 and categorize them by modality: F (user metadata), T (tweet and description content), and G (Twitter network structure).

For \textbf{BotRGCN}~\citep{feng2021botrgcn}, \textbf{BotGAT}~\citep{lei2022bic}, and \textbf{BotGCN}~\citep{feng2021botrgcn}, we concatenate our 512-dimensional GNN-generated user embeddings with the original user features, including numerical properties (\texttt{num\_prop}5), categorical properties (\texttt{cat\_prop}3), tweets (\texttt{tweet}), and user descriptions (\texttt{des}).

\section{Additional Experimental Details}

\subsection{Algorithm}
Algorithm~\ref{alg:algo} summarizes the training procedure of \textsc{ValueGraph}. 
At each epoch, we first sample seed users and construct an augmented user batch by retrieving similar and dissimilar users according to their inferred value profiles. 
The corresponding posts are encoded by the graph encoder, and user embeddings are obtained by aggregating post representations. 
The value-guided contrastive loss $\mathcal{L}_{\mathrm{user}}$ is then computed to align users with similar value profiles while separating dissimilar ones. 
To further regularize the global structure of the embedding space, we periodically perform $K$-means clustering every $X$ epochs and compute the clustering loss $\mathcal{L}_{\mathrm{cls}}$. 
The final objective combines both losses to update the encoder parameters.

\begin{algorithm}[t!]
\small
\caption{\textsc{ValueGraph} Training}
\label{alg:algo}
\begin{algorithmic}[1]
\Require Graph $\mathcal{G}=(\mathcal{U},\mathcal{V},\mathcal{E})$; foundation encoder $f$; similar-user sets $\mathcal{S}$; dissimilar-user sets $\mathcal{D}$; augmented users per seed $k$; temperature $\tau$; loss weight $\lambda$; number of clusters $K$; epochs $E$; clustering interval $X$
\Ensure Trained encoder $f$

\For{$e=1$ \textbf{to} $E$}
    \State Sample seed users $U_{\mathrm{seed}}\subset\mathcal{U}$
    \State Construct augmented user batch
    \[
    U^{+}_{\mathrm{seed}}
    \gets
    \bigcup_{u\in U_{\mathrm{seed}}}
    \left(
    \tilde{u}\overset{k}{\sim}\mathcal{S}(u)
    \cup
    \hat{u}\overset{k}{\sim}\mathcal{D}(u)
    \right)
    \]
    \State Collect posts $V \gets \bigcup_{u\in U^{+}_{\mathrm{seed}}}\phi(u)$
    \State Encode posts $z_v \gets f(v),\ \forall v\in V$
    \State Compute user embeddings
    \[
    z_u \gets \mathrm{MEAN}\bigl(\{z_v \mid v\in\phi(u)\}\bigr),
    \quad \forall u\in U^{+}_{\mathrm{seed}}
    \]
    \State Compute $\mathcal{L}_{\mathrm{user}}$ using Eq.~\ref{eq:luser}

    \If{$e \bmod X = 0$ \textbf{or} $e=1$}
        \State Run $K$-means on $\{z_u \mid u\in U^{+}_{\mathrm{seed}}\}$
        \State Compute $\mathcal{L}_{\mathrm{cls}}$ using Eq.~\ref{eq:lcls}
    \Else
        \State $\mathcal{L}_{\mathrm{cls}} \gets 0$
    \EndIf

    \State $\mathcal{L} \gets \mathcal{L}_{\mathrm{user}}+\lambda\mathcal{L}_{\mathrm{cls}}$
    \State Update $f$ by gradient descent on $\mathcal{L}$
\EndFor

\State \Return $f$
\end{algorithmic}
\end{algorithm}

\subsection{Hyperparameter Settings}\label{appendix:Hyperparameter Settings}

We report the hyperparameter configurations used in our experiments. Tables~\ref{tab:stance_detection_hyperparams} and~\ref{tab:twitter_both_hyperparams} and summarize the settings for stance detection and Twitter bot detection, respectively. Hyperparameters are tuned for each model and dataset. Additional search may further improve performance for specific embedding configurations.

\begin{table*}[htb]
  \centering
  \resizebox{\textwidth}{!}{%
    \begin{tabular}{@{} l c c c c c c @{}}
      \toprule
      Model                             & Hidden Dim & Learning Rate & Weight Decay & Dropout & Batch Size & Max Epoch \\
      \midrule
      GLAN (MT\_CSD)                & 512        & 1e-5          & 1e-4         & 0.2     & 32         & 100        \\
      BrLSTM (RumourEval19)         & 300        & 1e-4          & 1e-4         & 0.5     & 16         & 100        \\
      \bottomrule
    \end{tabular}%
  }
  \caption{Hyperparameter settings for stance detection models.}
  \label{tab:stance_detection_hyperparams}
\end{table*}

\begin{table*}[htb]
  \centering
  \resizebox{\textwidth}{!}{%
    \begin{tabular}{@{} l c c c c c c @{}}
      \toprule
      Model                             & Hidden Dim & Learning Rate & Weight Decay & Dropout & Batch Size & Max Epoch \\
      \midrule
      RoBERTa (T)                       & 128        & 1e-4          & 1e-5         & 0.5     & 96         & 60        \\
      RoBERTa (T,U)                     & 128        & 1e-4          & 1e-5         & 0.5     & 96         & 60        \\
      BotRGCN (T,G)                     & 256        & 5e-4          & 5e-5         & 0.5     & 128        & 300       \\
      BotRGCN (F,T,U,G)                 & 128        & 5e-4          & 1e-3         & 0.3     & 128        & 300       \\
      GAT (T,G)                         & 96         & 2e-4          & 5e-5         & 0.5     & 128        & 300       \\
      GAT (F,T,U,G)                     & 96         & 2e-4          & 5e-5         & 0.5     & 256        & 300       \\
      GCN (T,G)                         & 256        & 2e-4          & 1e-4         & 0.3     & 128        & 300       \\
      GCN (F,T,U,G)                     & 256        & 5e-4          & 1e-3         & 0.3     & 128        & 300       \\
      \bottomrule
    \end{tabular}%
  }
  \caption{Hyperparameter settings for Twitter bot detection models.}
  \label{tab:twitter_both_hyperparams}
\end{table*}


\subsection{Supplementary Results of the Ablation Experiments}\label{app:ablation_supp}

We report supplementary ablation results in Table~\ref{tab:ablation-sup}, which are omitted from the main text due to space limitations.

Table~\ref{tab:ablation-sup} evaluates \textsc{ValueGraph} under different training configurations. Using only Stage~1 yields limited performance across tasks. Adding the clustering loss $\mathcal{L}_{\text{cls}}$ generally improves task-aware discrimination, while incorporating the user-level consistency loss $\mathcal{L}_{\text{user}}$ further helps capture user interaction structure and value-aware relational information. Overall, the full model achieves the strongest and most stable performance across models and tasks, supporting the effectiveness and generalizability of \textsc{ValueGraph}.

\begin{table*}[htb!]
\centering
\begin{tabular}{l l l c c}
\toprule
\textbf{Task} & \textbf{Model/dataset} & \textbf{Configuration} &
\textbf{Accuracy} & 
\textbf{F1-score} \\
\midrule

\multirow{8}{*}{Stance Detection} 
& \multirow{4}{*}{GLAN (MT\_CSD)} 
    & Only Stage~1 & 41.2 & 40.7\\
&   & Stage~1 + $\mathcal{L}_{\text{cls}}$ & 50.4 & 34.3\\
&   & Stage~1 + $\mathcal{L}_{\text{user}}$ & 45.3 & 32.6\\
&   & Full & \textbf{63.0} & \textbf{58.0}\\

\cdashline{2-5}

& \multirow{4}{*}{BrLSTM (RumourEval19)} 
    & Only Stage~1 & 53.7 & 55.8\\
&   & Stage~1 + $\mathcal{L}_{\text{cls}}$ & 62.7 & 52.8\\
&   & Stage~1 + $\mathcal{L}_{\text{user}}$ & 53.9 & 45.7\\
&   & Full & \textbf{77.1} & \textbf{71.2}\\

\midrule










\multirow{16}{*}{Twitter Bot Detection}

& \multirow{4}{*}{RoBERTa (Twibot-22)}
    & Only Stage~1 & 63.6 & 68.9\\
&   & Stage~1 + $\mathcal{L}_{\text{cls}}$ & 64.9 & 67.8\\
&   & Stage~1 + $\mathcal{L}_{\text{user}}$ & 64.7 & 67.5\\
&   & Full & \textbf{65.9} & \textbf{68.9}\\

\cdashline{2-5}

& \multirow{4}{*}{BotGCN (Twibot-22)}
    & Only Stage~1 & 70.3 & 71.0\\
&   & Stage~1 + $\mathcal{L}_{\text{cls}}$ & 71.2 & 71.8\\
&   & Stage~1 + $\mathcal{L}_{\text{user}}$ & 70.9 & 71.7\\
&   & Full & \textbf{71.1} & \textbf{71.1}\\

\cdashline{2-5}

& \multirow{4}{*}{BotRGCN (Twibot-22)}
    & Only Stage~1 & 71.6 & 73.3\\
&   & Stage~1 + $\mathcal{L}_{\text{cls}}$ & 72.8 & 73.4\\
&   & Stage~1 + $\mathcal{L}_{\text{user}}$ & 72.5 & 73.1\\
&   & Full & \textbf{73.0} & \textbf{74.2}\\

\cdashline{2-5}

& \multirow{4}{*}{BotGAT (Twibot-22)}
    & Only Stage~1 & 73.1 & 70.3\\
&   & Stage~1 + $\mathcal{L}_{\text{cls}}$ & 72.9 & 72.6\\
&   & Stage~1 + $\mathcal{L}_{\text{user}}$ & 72.6 & 72.0\\
&   & Full & \textbf{74.7} & \textbf{73.0}\\

\bottomrule
\end{tabular}

\caption{Ablation results under different training configurations across two tasks: stance detection and Twitter bot detection. 
For Twitter bot detection, RoBERTa is evaluated in the UT setting, while BotGAT, BotGCN, and BotRGCN are evaluated in the FTUG setting.}
\label{tab:ablation-sup}
\end{table*}

\subsection{GPU Hours}\label{app:GPU_hours}

All models are run with fixed random seeds for reproducibility. Training and evaluation are conducted on a compute cluster with two NVIDIA H100 GPUs. The full suite of experiments consumes approximately 1,000 GPU hours, including two-stage \textsc{ValueGraph} pre-training and downstream evaluations under both baseline and \textsc{ValueGraph} configurations.

\textsc{ValueGraph} training consists of Stage 1, which uses textual and network signals to learn initial representations, and Stage 2, which incorporates inferred value signals to produce final user embeddings. These embeddings replace the corresponding conventional features in downstream evaluations.

\subsection{t-SNE Setting}\label{appendix:TSNE}
For consistency across models, we apply t-SNE to project high-dimensional embeddings into two dimensions, using \texttt{n\_components=2}, \texttt{perplexity=30}, \texttt{n\_iter=1000}, and \texttt{random\_state=42}. This configuration preserves local neighborhood structure while ensuring reproducibility.

\subsection{GPT-5.4 Prompt Templates}\label{appendix:gpt54_prompts}

We include GPT-5.4 as a text-only LLM baseline for stance detection and Twitter bot detection. GPT-5.4 receives only textual inputs, including the target, user posts, thread text, or profile text when available; it does not receive explicit reply-graph adjacency or message-passing structure. We use deterministic decoding with temperature set to 0. The prompts used for the two tasks are shown below.

\paragraph{Stance Detection Prompt.}
For stance detection, the model is given the target, the source post or thread context, and the user's available posts. It is asked to predict one of the task labels. In our default setting (\texttt{text\_mode=leaf}), MT-CSD uses the discussion topic together with the leaf comment, and RumourEval19 uses the reply text only.

\begin{quote}
\small
\textbf{Task:} Determine the user's stance toward the given target.

\textbf{Target:} \{target\}

\textbf{Conversation context:} \{thread\_text\}

\textbf{User posts:} \{user\_posts\}

Choose exactly one label from: \texttt{Favor}, \texttt{Against}, \texttt{None}.

Return only the label.
\end{quote}

In practice, for MT-CSD we instantiate \texttt{\{target\}} with the discussion topic and use the leaf comment as the textual input; for RumourEval19, we use the reply text as the input and omit separate thread or user-history fields in the default configuration. The implemented MT-CSD prompt is:

\begin{quote}
\small
You are annotating stance in a social-media discussion about the topic: \{topic\}.

Classify the STANCE of the following comment toward the discussion target.
Use exactly one label:
- favor: supports/agrees with the target stance
- against: opposes/disagrees with the target stance
- none: neutral, unrelated, or unclear stance

Comment:
\{text\}

Reply with exactly one word: favor, against, or none.
\end{quote}

For RumourEval19, we replace the label set with the dataset-specific labels:

\begin{quote}
\small
Choose exactly one label from: \texttt{Support}, \texttt{Deny}, \texttt{Query}, \texttt{Comment}.

Return only the label.
\end{quote}

The implemented RumourEval19 prompt is:

\begin{quote}
\small
You are annotating stance toward a rumor in social media.

Classify the STANCE of the following reply (SDQC scheme).
Use exactly one label:
- support: supports/agrees the rumor is true
- deny: refutes or disagrees with the rumor
- query: asks for evidence or clarification
- comment: neutral or unrelated to rumor veracity

Reply text:
\{text\}

Reply with exactly one word: support, deny, query, or comment.
\end{quote}

\paragraph{Twitter Bot Detection Prompt.}
For bot detection, the model is given the user's profile text and sampled posts. It is asked to classify the account as human or bot. In our TwiBot experiments, profile text is not available; we therefore provide only the user's test-split posts, aggregated at the account level.

\begin{quote}
\small
\textbf{Task:} Determine whether the following Twitter account is operated by a human or a bot.

\textbf{User profile:} \{profile\_text\}

\textbf{User posts:} \{sampled\_posts\}

Choose exactly one label from: \texttt{Human}, \texttt{Bot}.

Return only the label.
\end{quote}

In practice, we set \texttt{\{profile\_text\}} to empty and construct \texttt{\{sampled\_posts\}} by concatenating all available tweets from the same user in the test split, separated by \texttt{\textbackslash n---\textbackslash n}. The implemented prompt is:

\begin{quote}
\small
You are detecting whether a Twitter/X account is operated by a human or a bot.

Read the following tweets posted by ONE account (may be truncated). Classify the account type.
Use exactly one label:
- human: likely a real person
- bot: likely automated, spam, or bot-like

Tweets from this account:
\{text\}

Reply with exactly one word: human or bot.
\end{quote}

For all prompts, long inputs are truncated to fit the model context window while preserving the target, profile text, source post, and the most recent or most relevant user posts. In our implementation, the truncation limit is 12{,}000 characters. No examples from the test set are included in the prompt.

%% file: aaai2027.bib
@inproceedings{10.5555/3061053.3061153,
author = {Ma, Jing and Gao, Wei and Mitra, Prasenjit and Kwon, Sejeong and Jansen, Bernard J. and Wong, Kam-Fai and Cha, Meeyoung},
title = {Detecting rumors from microblogs with recurrent neural networks},
year = {2016},
booktitle = {IJCAI},
}

@article{trager2022moral,
  title={The moral foundations reddit corpus},
  author={Trager, Jackson and Ziabari, Alireza S and Davani, Aida Mostafazadeh and Golazizian, Preni and Karimi-Malekabadi, Farzan and Omrani, Ali and Li, Zhihe and Kennedy, Brendan and Reimer, Nils Karl and Reyes, Melissa and others},
  journal={arXiv:2208.05545},
  year={2022}
}

@Inproceedings{moralbert,
author = {Preniqi, Vjosa and Ghinassi, Iacopo and Ive, Julia and Saitis, Charalampos and Kalimeri, Kyriaki},
title = {MoralBERT: A Fine-Tuned Language Model for Capturing Moral Values in Social Discussions},
year = {2024},
booktitle = {GoodIT},
}

@inproceedings{nguyen2024measuring,
  title={Measuring moral dimensions in social media with mformer},
  author={Nguyen, Tuan Dung and Chen, Ziyu and Carroll, Nicholas George and Tran, Alasdair and Klein, Colin and Xie, Lexing},
  booktitle={ICWSM},
  year={2024}
}

@inproceedings{guo2023data,
  title={A data fusion framework for multi-domain morality learning},
  author={Guo, Siyi and Mokhberian, Negar and Lerman, Kristina},
  booktitle={ICWSM},
  year={2023}
}

@inproceedings{zangari2025me2,
  title={ME2-BERT: Are Events and Emotions what you need for Moral Foundation Prediction?},
  author={Zangari, Lorenzo and Greco, Candida M and Picca, Davide and Tagarelli, Andrea},
  booktitle={COLING},
  year={2025}
}

@inproceedings{ouyang2022training,
author = {Ouyang, Long and Wu, Jeff and Jiang, Xu and Almeida, Diogo and Wainwright, Carroll L. and Mishkin, Pamela and Zhang, Chong and Agarwal, Sandhini and Slama, Katarina and Ray, Alex and Schulman, John and Hilton, Jacob and Kelton, Fraser and Miller, Luke and Simens, Maddie and Askell, Amanda and Welinder, Peter and Christiano, Paul and Leike, Jan and Lowe, Ryan},
title = {Training language models to follow instructions with human feedback},
year = {2022},
booktitle = {NeurIPS},
}

@inproceedings{christiano2017rlhf,
     author = {Christiano, Paul F and Leike, Jan and Brown, Tom and Martic, Miljan and Legg, Shane and Amodei, Dario},
     booktitle = {NIPS},
     title = {Deep Reinforcement Learning from Human Preferences},
     year = {2017}
}

@inproceedings{pan2019social,
  title     = {Social Media-based User Embedding: A Literature Review},
  author    = {Pan, Shimei and Ding, Tao},
  booktitle = {IJCAI},
  year      = {2019},
}

@inproceedings{
Hu2020Strategies,
title={Strategies for Pre-training Graph Neural Networks},
author={Weihua Hu and Bowen Liu and Joseph Gomes and Marinka Zitnik and Percy Liang and Vijay Pande and Jure Leskovec},
booktitle={ICLR},
year={2020},
}

@conference{you2020graph,
  title={Graph contrastive learning with augmentations},
  author={You, Yuning and Chen, Tianlong and Sui, Yongduo and Chen, Ting and Wang, Zhangyang and Shen, Yang},
  booktitle={NeurIPS},
  year={2020}
}

@inproceedings{10.1145/3543507.3583379,
author = {Hou, Zhenyu and He, Yufei and Cen, Yukuo and Liu, Xiao and Dong, Yuxiao and Kharlamov, Evgeny and Tang, Jie},
title = {GraphMAE2: A Decoding-Enhanced Masked Self-Supervised Graph Learner},
year = {2023},
booktitle = {WWW},
}

@inproceedings{benton-etal-2016-learning,
    title = "Learning Multiview Embeddings of {T}witter Users",
    author = "Benton, Adrian  and
      Arora, Raman  and
      Dredze, Mark",
    booktitle = "ACL",
    year = "2016",
}

@inproceedings{ding-etal-2017-multi,
    title = "Multi-View Unsupervised User Feature Embedding for Social Media-based Substance Use Prediction",
    author = "Ding, Tao  and
      Bickel, Warren K.  and
      Pan, Shimei",
    booktitle = "EMNLP",
    year = "2017",
}

@inproceedings{rahimi-etal-2015-twitter,
    title = "{T}witter User Geolocation Using a Unified Text and Network Prediction Model",
    author = "Rahimi, Afshin  and
      Cohn, Trevor  and
      Baldwin, Timothy",
    booktitle = "ACL-IJCNLP",
    year = "2015",
}

@inproceedings{do2018twitter,
  title={Twitter user geolocation using deep multiview learning},
  author={Do, Tien Huu and Nguyen, Duc Minh and Tsiligianni, Evaggelia and Cornelis, Bruno and Deligiannis, Nikos},
  booktitle={ICASSP},
  year={2018},
}

@inproceedings{zhang2017user,
  title={User profile preserving social network embedding},
  author={Zhang, Daokun and Yin, Jie and Zhu, Xingquan and Zhang, Chengqi},
  booktitle={IJCAI},
  year={2017}
}

@inproceedings{10.1145/3178876.3186026,
author = {Zhang, Wei and Wang, Wen and Wang, Jun and Zha, Hongyuan},
title = {User-guided Hierarchical Attention Network for Multi-modal Social Image Popularity Prediction},
year = {2018},
booktitle = {WWW},
}

@inproceedings{ribeiro2018characterizing,
  title={Characterizing and detecting hateful users on twitter},
  author={Ribeiro, Manoel and Calais, Pedro and Santos, Yuri and Almeida, Virg{\'\i}lio and Meira Jr, Wagner},
  booktitle={ICWSM},
  year={2018}
}

@inproceedings{Perozzi2014DeepWalkOL,
  title={DeepWalk: online learning of social representations},
  author={Bryan Perozzi and Rami Al-Rfou and Steven Skiena},
  booktitle={SIGKDD},
  year={2014}
}

@inproceedings{Sun2020MultiStageSL,
  title={Multi-Stage Self-Supervised Learning for Graph Convolutional Networks},
  author={Ke Sun and Zhanxing Zhu and Zhouchen Lin},
  booktitle={AAAI},
  year={2020},
}

@inproceedings{Grover2016node2vecSF,
  title={node2vec: Scalable Feature Learning for Networks},
  author={Aditya Grover and Jure Leskovec},
  booktitle={SIGKDD},
  year={2016}
}

@inproceedings{Donnat2018LearningSN,
  title={Learning Structural Node Embeddings via Diffusion Wavelets},
  author={Claire Donnat and Marinka Zitnik and David Hallac and Jure Leskovec},
  booktitle={SIGKDD},
  year={2018}
}

@inproceedings{Zhang2019ProNEFA,
  title={ProNE: Fast and Scalable Network Representation Learning},
  author={Jie Zhang and Yuxiao Dong and Yan Wang and Jie Tang and Ming Ding},
  booktitle={IJCAI},
  year={2019}
}

@inproceedings{Tang2015LINELI,
  title={LINE: Large-scale Information Network Embedding},
  author={Jian Tang and Meng Qu and Mingzhe Wang and Ming Zhang and Jun Yan and Qiaozhu Mei},
  booktitle={WWW},
  year={2015}
}

@inproceedings{Zhao2021DataAF,
  title={Data Augmentation for Graph Neural Networks},
  author={Tong Zhao and Yozen Liu and Leonardo Neves and Oliver J. Woodford and Meng Jiang and Neil Shah},
  booktitle={AAAI},
  year={2021}
}

@inproceedings{hassani2020contrastive,
  title={Contrastive multi-view representation learning on graphs},
  author={Hassani, Kaveh and Khasahmadi, Amir Hosein},
  booktitle={ICML},
  year={2020},
}

@article{Hafidi2020GraphCLCS,
  title={GraphCL: Contrastive Self-Supervised Learning of Graph Representations},
  author={Hakim Hafidi and Mounir Ghogho and Philippe Ciblat and Ananthram Swami},
  journal={ArXiv},
  year={2020},
  volume={abs/2007.08025}
}

@inproceedings{Sun2020InfoGraphUA,
  title={InfoGraph: Unsupervised and Semi-supervised Graph-Level Representation Learning via Mutual Information Maximization},
 author ={Fan-Yun Sun and Jordan Hoffmann and Jian Tang},
  booktitle={ICLR},
  year={2020},
}

@article{Zhu2020DeepGC,
  title={Deep Graph Contrastive Representation Learning},
  author={Yanqiao Zhu and Yichen Xu and Feng Yu and Q. Liu and Shu Wu and Liang Wang},
  journal={ArXiv},
  year={2020},
  volume={abs/2006.04131}
}

@inproceedings{Qiu2020GCCGC,
  title={GCC: Graph Contrastive Coding for Graph Neural Network Pre-Training},
  author={Jiezhong Qiu and Qibin Chen and Yuxiao Dong and Jing Zhang and Hongxia Yang and Ming Ding and Kuansan Wang and Jie Tang},
  booktitle={SIGKDD},
  year={2020}
}

@inproceedings{Hu2020GPTGNNGP,
  title={GPT-GNN: Generative Pre-Training of Graph Neural Networks},
  author={Ziniu Hu and Yuxiao Dong and Kuansan Wang and Kai-Wei Chang and Yizhou Sun},
  booktitle={SIGKDD},
  year={2020}
}

@inproceedings{Lu2021LearningTP,
  title={Learning to Pre-train Graph Neural Networks},
  author={Yuanfu Lu and Xunqiang Jiang and Yuan Fang and Chuan Shi},
  booktitle={AAAI},
  year={2021}
}

@article{pheme,
  author = {Zubiaga, Arkaitz and Caba Heilbron, Andrea and Liakata, Maria and Procter, Rob and Tolmie, Peter and Bontcheva, Kalina},
  title = {PHEME: A dataset for fine-grained stance detection},
  journal = {arXiv:1610.07363},
  year = {2016},
}

@inproceedings{twitter_covid,
    title = "Detect Rumors in Microblog Posts for Low-Resource Domains via Adversarial Contrastive Learning",
    author = "Lin, Hongzhan  and
      Ma, Jing  and
      Chen, Liangliang  and
      Yang, Zhiwei  and
      Cheng, Mingfei  and
      Guang, Chen",
    booktitle = "Findings of NAACL",
    year = "2022",
}

@article{haidt2007moral,
  title={The moral mind: How five sets of innate intuitions guide the development of many culture-specific virtues, and perhaps even modules},
  author={Haidt, Jonathan and Joseph, Craig and others},
  journal={The innate mind},
  volume={3},
  pages={367--391},
  year={2007}
}

@article{schwartz2012refining,
  title     = {Refining the Theory of Basic Individual Values},
  author    = {Schwartz, Shalom H. and Cieciuch, Jan and Vecchione, Michele and Davidov, Eldad and Fischer, Ronald and Beierlein, Constanze and Ramos, Alice and Verkasalo, Markku and L{\"o}nnqvist, Jan-Erik and Demirutku, Kadriye and others},
  journal   = {Journal of Personality and Social Psychology},
  volume    = {103},
  number    = {4},
  pages     = {663},
  year      = {2012},
}

@book{hofstede2001culture,
  title     = {Culture's Consequences: Comparing Values, Behaviors, Institutions and Organizations Across Nations},
  author    = {Hofstede, Geert},
  year      = {2001},
  publisher = {Sage Publications}
}

@article{ARAQUE2020105184,
title = {MoralStrength: Exploiting a moral lexicon and embedding similarity for moral foundations prediction},
journal = {Knowledge-Based Systems},
volume = {191},
pages = {105184},
year = {2020},
author = {Oscar Araque and Lorenzo Gatti and Kyriaki Kalimeri},
}

@article{haidt_all,
author = {Graham, Jesse and Haidt, Jonathan and Nosek, Brian},
year = {2009},
month = {06},
pages = {1029-46},
title = {Liberals and Conservatives Rely on Different Sets of Moral Foundations},
volume = {96},
journal = {Journal of personality and social psychology},
}

@inproceedings{johnson2018classification,
  title={Classification of moral foundations in microblog political discourse},
  author={Johnson, Kristen and Goldwasser, Dan},
  booktitle={ACL},
  year={2018}
}

@article{mooijman2018moralization,
  title={Moralization in social networks and the emergence of violence during protests},
  author={Mooijman, Marlon and Hoover, Joe and Lin, Ying and Ji, Heng and Dehghani, Morteza},
  journal={Nature human behaviour},
  volume={2},
  number={6},
  pages={389--396},
  year={2018},
}

@inproceedings{mokhberian2020moral,
  title={Moral framing and ideological bias of news},
  author={Mokhberian, Negar and Abeliuk, Andr{\'e}s and Cummings, Patrick and Lerman, Kristina},
  booktitle={SocInfo},
  year={2020},
}

@inproceedings{lin2018acquiring,
  title={Acquiring background knowledge to improve moral value prediction},
  author={Lin, Ying and Hoover, Joe and Portillo-Wightman, Gwenyth and Park, Christina and Dehghani, Morteza and Ji, Heng},
  booktitle={ASONAM},
  year={2018},
}

@InProceedings{zhang2024enhancingstanceclassificationsocial,
author="Zhang, Hong
and Nguyen, Quoc-Nam
and Bhattacharya, Prasanta
and Gao, Wei
and Wong, Liang Ze
and Loh, Brandon Siyuan
and Simons, Joseph J. P.
and An, Jisun",
title="Enhancing Stance Classification on Social Media Using Quantified Moral Foundations",
booktitle="ASONAM",
year="2025",
}

@inproceedings{li2021self,
  title={Self-supervised learning with kernel dependence maximization},
  author={Li, Yazhe and Pogodin, Roman and Sutherland, Danica J and Gretton, Arthur},
  booktitle={NeurIPS},
  year={2021}
}

@inproceedings{nguyen-etal-2020-bertweet,
    title = "{BERT}weet: A pre-trained language model for {E}nglish Tweets",
    author = "Nguyen, Dat Quoc  and
      Vu, Thanh  and
      Tuan Nguyen, Anh",
    booktitle = "EMNLP",
    year = "2020",
}

@inproceedings{warner2024smarterbetterfasterlonger,
    title = "Smarter, Better, Faster, Longer: A Modern Bidirectional Encoder for Fast, Memory Efficient, and Long Context Finetuning and Inference",
    author = {Warner, Benjamin  and
      Chaffin, Antoine  and
      Clavi{\'e}, Benjamin  and
      Weller, Orion  and
      Hallstr{\"o}m, Oskar  and
      Taghadouini, Said  and
      Gallagher, Alexis  and
      Biswas, Raja  and
      Ladhak, Faisal  and
      Aarsen, Tom  and
      Adams, Griffin Thomas  and
      Howard, Jeremy  and
      Poli, Iacopo},
    booktitle = "ACL",
    year = "2025",
}

@inproceedings{gao2021simcse,
   title={{SimCSE}: Simple Contrastive Learning of Sentence Embeddings},
   author={Gao, Tianyu and Yao, Xingcheng and Chen, Danqi},
   booktitle={EMNLP},
   year={2021}
}

@inproceedings{gorrell-etal-2019-semeval,
    title = "{S}em{E}val-2019 Task 7: {R}umour{E}val, Determining Rumour Veracity and Support for Rumours",
    author = "Gorrell, Genevieve  and
      Kochkina, Elena  and
      Liakata, Maria  and
      Aker, Ahmet  and
      Zubiaga, Arkaitz  and
      Bontcheva, Kalina  and
      Derczynski, Leon",
    booktitle = "SemEval",
    year = "2019",
}

@article{kipf2016semi,
  title={Semi-supervised classification with graph convolutional networks},
  author={Kipf, Thomas N and Welling, Max},
  journal={arXiv:1609.02907},
  year={2016}
}

@article{velickovic2017graph,
  title={Graph attention networks},
  author={Velickovic, Petar and Cucurull, Guillem and Casanova, Arantxa and Romero, Adriana and Lio, Pietro and Bengio, Yoshua and others},
  journal={stat},
  volume={1050},
  number={20},
  pages={10--48550},
  year={2017}
}

@inproceedings{yun2019graph,
  title={Graph transformer networks},
  author={Yun, Seongjun and Jeong, Minbyul and Kim, Raehyun and Kang, Jaewoo and Kim, Hyunwoo J},
  booktitle={NeurIPS},
  year={2019}
}

@inproceedings{10.5555/3600270.3602825,
author = {Feng, Shangbin and Tan, Zhaoxuan and Wan, Herun and Wang, Ningnan and Chen, Zilong and Zhang, Binchi and Zheng, Qinghua and Zhang, Wenqian and Lei, Zhenyu and Yang, Shujie and Feng, Xinshun and Zhang, Qingyue and Wang, Hongrui and Liu, Yuhan and Bai, Yuyang and Wang, Heng and Cai, Zijian and Wang, Yanbo and Zheng, Lijing and Ma, Zihan and Li, Jundong and Luo, Minnan},
title = {TwiBot-22: towards graph-based twitter bot detection},
year = {2022},
booktitle = {NeurIPS},
}

@inproceedings{feng2021botrgcn,
  title     = {BotRGCN: Twitter Bot Detection with Relational Graph Convolutional Networks},
  author    = {Shangbin Feng and Herun Wan and Ningnan Wang and Minnan Luo},
  booktitle = {ASONAM},
  year      = {2021},
}

@article{lei2022bic,
  title   = {BIC: Twitter Bot Detection with Text-Graph Interaction and Semantic Consistency},
  author  = {Zhenyu Lei and Herun Wan and Wenqian Zhang and Shangbin Feng and Zilong Chen and Jundong Li and Qinghua Zheng and Minnan Luo},
  journal = {arXiv:2208.08320},
  year    = {2022},
  url     = {https://arxiv.org/abs/2208.08320}
}

@inproceedings{10.1145/3351095.3372879,
author = {Ribeiro, Manoel Horta and Ottoni, Raphael and West, Robert and Almeida, Virg\'{\i}lio A. F. and Meira, Wagner},
title = {Auditing radicalization pathways on YouTube},
year = {2020},
booktitle = {FAT},
}

@article{liu2019robertarobustlyoptimizedbert,
  title={Roberta: A robustly optimized bert pretraining approach},
  author={Liu, Yinhan and Ott, Myle and Goyal, Naman and Du, Jingfei and Joshi, Mandar and Chen, Danqi and Levy, Omer and Lewis, Mike and Zettlemoyer, Luke and Stoyanov, Veselin},
  journal={arXiv:1907.11692},
  year={2019}
}

@inproceedings{zeng-gao-2022-early,
    title = "{E}arly Rumor Detection Using Neural {H}awkes Process with a New Benchmark Dataset",
    author = "Zeng, Fengzhu  and
      Gao, Wei",
    booktitle = "NAACL",
    year = "2022",
}

@inproceedings{niu-etal-2024-challenge,
    title = "A Challenge Dataset and Effective Models for Conversational Stance Detection",
    author = "Niu, Fuqiang  and
      Yang, Min  and
      Li, Ang  and
      Zhang, Baoquan  and
      Peng, Xiaojiang  and
      Zhang, Bowen",
    booktitle = "LREC-COLING",
    year = "2024",
}

@inproceedings{alkhatib-2020-personal,
  title={Exploiting Personal Characteristics of Debaters for Predicting Persuasiveness},
  author={Khalid Al-Khatib and Michael V{\"o}lske and Shahbaz Syed and Nikolay Kolyada and Benno Stein},
  booktitle={ACL},
  year={2020},
}

@inproceedings{durmus-2018-beliefs,
  title={Exploring the Role of Prior Beliefs for Argument Persuasion},
  author={Esin Durmus and Claire Cardie},
  booktitle={NAACL},
  year={2018},
}

@article{zhao2024graphprompt,
  title={Graph Prompt Learning for Stance Detection},
  author={Zhao, Yifan and Xue, Zhihan and Zhang, Jipeng and Wei, Furu and Che, Wanxiang and Liu, Ting},
  journal={arXiv:2403.11145},
  year={2024}
}

@inproceedings{ijcai2022p533,
  title={Multi-Target Stance Detection with Bi-Directional Recursive Encoding and Classification},
  author={Yue, Haijun and He, Yulan and Shu, Kai and Liu, Huan},
  booktitle={IJCAI},
  year={2022},
}

@inproceedings{pick2022stem,
  title={Stem: unsupervised structural embedding for stance detection},
  author={Pick, Ron Korenblum and Kozhukhov, Vladyslav and Vilenchik, Dan and Tsur, Oren},
  booktitle={AAAI},
  year={2022}
}

@misc{yuan2019jointlyembeddinglocalglobal,
      title={Jointly embedding the local and global relations of heterogeneous graph for rumor detection}, 
      author={Chunyuan Yuan and Qianwen Ma and Wei Zhou and Jizhong Han and Songlin Hu},
      year={2019},
      eprint={1909.04465},
      archivePrefix={arXiv},
      primaryClass={cs.CL},
      url={https://arxiv.org/abs/1909.04465}, 
}

@article{T5,
author = {Raffel, Colin and Shazeer, Noam and Roberts, Adam and Lee, Katherine and Narang, Sharan and Matena, Michael and Zhou, Yanqi and Li, Wei and Liu, Peter J.},
title = {Exploring the limits of transfer learning with a unified text-to-text transformer},
year = {2020},
volume = {21},
number = {1},
journal = {J. Mach. Learn. Res.},
}

@incollection{GRAHAM201355,
title = {Chapter Two - Moral Foundations Theory: The Pragmatic Validity of Moral Pluralism},
editor = {Patricia Devine and Ashby Plant},
series = {Advances in Experimental Social Psychology},
publisher = {Academic Press},
volume = {47},
pages = {55-130},
year = {2013},
issn = {0065-2601},
doi = {https://doi.org/10.1016/B978-0-12-407236-7.00002-4},
url = {https://www.sciencedirect.com/science/article/pii/B9780124072367000024},
author = {Jesse Graham and Jonathan Haidt and Sena Koleva and Matt Motyl and Ravi Iyer and Sean P. Wojcik and Peter H. Ditto}
}

@InProceedings{Caron_2018_ECCV,
author = {Caron, Mathilde and Bojanowski, Piotr and Joulin, Armand and Douze, Matthijs},
title = {Deep Clustering for Unsupervised Learning of Visual Features},
booktitle = {Proceedings of the European Conference on Computer Vision (ECCV)},
month = {September},
year = {2018}
}
